\documentclass{scrartcl}
\usepackage{amsmath,amssymb}
\usepackage{amsthm}
\newtheorem{theorem}{Theorem}
\newtheorem{lemma}[theorem]{Lemma}

\newtheorem{remark}{Remark}
\usepackage{graphicx}
\usepackage{xcolor}
\usepackage{tikz}
\usetikzlibrary{positioning}
\usepackage{microtype}
\usepackage{hyperref}
\usepackage{booktabs}
\usepackage{adjustbox}
\usepackage{multirow}
\usepackage{natbib}
\newcommand{\R}{\mathbb{R}}
\newcommand{\indicator}[1]{\mathbf{1}_{#1}}
\newcommand{\probability}[1]{\operatorname{P}\left[#1\right]}
\newcommand{\expectation}[1]{\operatorname{\mathbb{E}}\left[#1\right]}
\newcommand{\pto}{\stackrel{\mathrm{p}}{\to}}
\newcommand{\randomfeatures}{\mathbf{X}}
\newcommand{\features}{\mathbf{x}}
\newcommand{\forward}{\mathrm{f}}
\newcommand{\backward}{\mathrm{b}}
\newcommand{\validaton}{\mathrm{val}}
\newcommand{\node}{e}
\newcommand{\parent}{\mathrm{parent}}
\newcommand{\leftchild}{L}
\newcommand{\rightchild}{R}
\newcommand{\indexTraining}{J}
\newcommand{\indexValidation}{I}
\newcommand{\noise}{\varepsilon}
\newcommand{\abs}[1]{\left\lvert #1 \right\rvert}
\newcommand{\signedHet}{\mathrm{Het}^{\pm}}
\newcommand{\signedBias}{\mathrm{Bias}^{\pm}}
\newcommand{\estSignedHet}{\widehat{\mathrm{Het}}^{\pm}}
\newcommand{\estSignedBias}{\widehat{\mathrm{Bias}}^{\pm}}
\newcommand{\estHet}{\widehat{\mathrm{Het}}}
\newcommand{\estBias}{\widehat{\mathrm{Bias}}}
\newcommand{\estHetV}{\widehat{\mathrm{Het}}^v}
\newcommand{\estBiasV}{\widehat{\mathrm{Bias}}^v}
\newcommand{\sgn}{\operatorname{sgn}}
\newcommand{\unobserved}[1]{\small{\textit{#1}}}
\newcommand{\observed}[1]{\textbf{#1}}
\definecolor{purplesignal}{HTML}{785ef0}
\definecolor{bluesignal}{HTML}{648fff}
\newcommand{\firstbest}[1]{\textbf{\underline{\color{bluesignal} #1}}}
\newcommand{\secondbest}[1]{\textbf{{\color{purplesignal} #1}}}

\title{Splitting the Difference: Interpretable Causal Forests for Treatment Effect Heterogeneity and Bias}
\author{Nicolas Alexander Ihlo\thanks{nicolas.ihlo@ur.de}, Merle Behr\thanks{merle.behr@ur.de}\\ Faculty of Informatics and Data Science\\ University of Regensburg, Germany}
\date{September 15, 2026}

\begin{document}
\maketitle

\begin{abstract}
  In various fields, such as medicine and marketing, accurately predicting individual treatment effects holds significant promise. However, achieving reliable predictions alone is often insufficient for making informed decisions; it is equally important to understand why the treatment effect is higher for some individuals than for others. To address this two-fold challenge of prediction and interpretation, we introduce an algorithm based on decision trees and random forests for estimating individual treatment effects. Our algorithm is simple: it operates exactly like a standard random forest, but with a different splitting criterion, and requires no additional workarounds such as double machine learning or orthogonalization as used in Generalized random forests. It handles observational studies with varying treatment propensities without requiring separate estimation of the full propensity function. This is achieved by combining two splitting criteria---one targeting heterogeneity in the treatment effect, the other targeting bias correction for the average treatment effect---which together improve split point selection and automatically distinguish confounders from features responsible for heterogeneity. As a result, interpretation follows directly from the fitted tree structure itself, that is, from which features the trees split on and with which split statistics, without requiring separate post-hoc analysis. For the theoretical analysis of this algorithm, we consider a change point model with step functions for potential outcomes and treatment propensity and provide insights into the theoretical underpinnings of our approach. Simulation studies show that our simple algorithm achieves comparable, and often better, prediction accuracy than existing methods, while substantially improving interpretability.
\end{abstract}

\emph{Keywords:} causal inference, statistical machine learning, interpretable machine learning, heterogeneous treatment effect estimation

\section{Introduction}
Estimating treatment effects, i.e., changes in an outcome due to some intervention, is of interest across many fields \citep{manski_identification_1993,imbens_causal_2015,angrist_mastering_2015}, including medicine, marketing, public policy, and economics. Beyond average effects, many applications benefit from individual-level estimates: for example, estimating the effectiveness of a drug for a specific patient enables individualized medicine. For such estimates to be useful in practice, however, accurate predictions are often not enough by themselves; understanding why the treatment effect differs across individuals, i.e., interpreting the estimation method, is equally important. In this paper, we focus on the estimation of individual treatment effects with a particular emphasis on interpretability.

In many settings, treatment effects cannot be estimated from randomized experiments and one has to rely on observational data instead. This introduces the problem of confounding: features that influence both the outcome and the treatment propensity, and which can bias treatment effect estimates even when they do not affect the treatment effect itself \citep{hernan_causal_2020}. Crucially, this means a feature can appear relevant to treatment effect estimation for two entirely different reasons: either because it is a confounder that must be corrected for, or because it genuinely drives treatment effect heterogeneity. Distinguishing between these two cases is precisely the interpretation we are interested in. Throughout this paper, we assume that all confounders are observed, i.e., we rule out hidden confounding, and focus on the estimation of individual treatment effects under observed confounding.

Machine learning (ML) methods are well suited to this task, as they can learn flexible structures directly from data with little manual adjustment. General supervised learning methods, originally developed for classification or regression, can be adapted for treatment effect estimation, for example via metalearners \citep{kunzel_metalearners_2019} and double/debiased machine learning (DML) \citep{chernozhukov_doubledebiased_2018}; other methods are purpose-built for causal inference, building on neural networks, for example TarNet \citep{shalit_estimating_2017}, DragonNet \citep{shi_adapting_2019}, and RA-Net \citep{curth2021nonparametric}. However, many of the most powerful ML methods act as black boxes and offer little insight into their prediction mechanism, even though such insight is often required for interpretation and downstream decision-making. 

Random forest (RF) \citep{breiman_random_2001}, originally developed for supervised learning but since extended to causal inference as well, offers a compromise between these two goals: it achieves state-of-the-art prediction accuracy, particularly on tabular data \citep{shmuel_comprehensive_2025}, while remaining, to a certain extent, interpretable through its underlying tree structure. This structure has, for instance, been used to gain insight into the mechanism of the fitted forest, e.g., by feature importance statistics such as mean decrease in impurity (MDI) \citep{breiman_random_2001}, as well as several extensions and other approaches, e.g.,  MDI+ \citep{agarwal_integrating_2025}, iRF \citep{basu_2018}, LSSFind \citep{behr_provable_2022}, and TreeSHAP \citep{lundberg_local_2020}.

Several RF-type algorithms have been proposed for estimating treatment effects, and it is this line of work that we build on and extend in this paper; we review these approaches in the following subsection. Our main motivation for building on RF, rather than a black-box method such as a neural network, is its interpretability, and accordingly, our goal is to preserve as much of this interpretability as possible in the RF-type algorithm we propose.

\subsection{Previous Tree-based Methods for Causal Inference}
A popular adaptation of RF is causal forest \citep{wager_estimation_2018}, which itself uses a simplified version of the causal-inference-specific adaptation of decision trees, causal tree, introduced by \citet{athey_recursive_2016}. While causal forest proposes two possible procedures, we focus here on procedure 1 (double sample trees), as only it incorporates outcomes in the construction of the trees and is therefore able to model treatment effect heterogeneity via the tree structure. Both causal tree and causal forest build trees by selecting splits that maximize the variance of treatment effect predictions, a criterion derived from an analogy to mean squared error (MSE) minimization in regression (see Section~\ref{Sec:splitting-rule} for details); however, this analogy relies on properties of the regression setting that do not hold in general for causal inference, and the resulting susceptibility to confounding has been noted previously, e.g., by \citet{athey_generalized_2019}. Notably, among the tree-based methods discussed below, causal forest is the only one that preserves the simple structure of the original RF: the individual treatment effect is estimated directly as the average prediction of a single ensemble of decision trees, each of which remains interpretable on its own.

Other adaptations of tree-based methods for causal inference typically sacrifice this simplicity. Metalearners \citep{kunzel_metalearners_2019}, for instance, do not estimate the treatment effect directly with a single forest, but instead combine separate forests fit to the outcomes under treatment and under control, so that the treatment effect estimate is no longer the direct output of one interpretable ensemble. Generalized random forest (GRF) \citep{athey_generalized_2019}, used as an alternative implementation of causal forest, instead modifies the splitting criterion via gradient-based approximations to an estimating equation and a correspondingly changed prediction mechanism, while additionally incorporating an orthogonalization step to reduce the influence of confounders. This orthogonalization step is similar in spirit to DML, reflected in the name ``CausalForestDML'' used for its implementation in the popular package EconML \citep{econml}. While effective at reducing confounding bias, this two-step procedure again departs from the direct, single-ensemble structure that makes RF interpretable in the first place. A method to improve feature importance for treatment effect heterogeneity in GRF was recently developed by \citet{benard_2025}, but requires additional post-hoc steps which are computationally costly, and, more importantly, does not address our goal of interpreting the fitted tree structure itself. Further adaptations include orthogonal random forest \citep{oprescu_orthogonal_2019}; see \citet{jiang_short_2021} for an overview and comparison of these methods.

Causal forest is therefore the only method that retains the interpretable, single-ensemble structure of RF for individual treatment effect estimation. However, as noted above, it does not explicitly account for confounders during split selection, so splits driven purely by confounding cannot be distinguished from splits that reflect genuine treatment effect heterogeneity, undermining the very interpretability its tree structure would otherwise offer. In this paper, we propose a modification of RF for causal inference, denoted as IntCF (interpretable causal forest), that follows the general idea of causal forest but uses a different splitting criterion, one that explicitly accounts for confounders at the split selection step. This direct inclusion of both heterogeneity and confounder detection in the tree structure opens the possibility for improved interpretation, especially for distinguishing features based on how they affect the treatment effect prediction, while preserving the simple, single-ensemble structure that makes RF interpretable.

\subsection{Confounding versus Heterogeneity: an Illustrative Example}
In the following, we provide an illustrative example, which demonstrates why the original splitting criterion of causal forest \citep{wager_estimation_2018} fails to result in interpretable tree structures under confounding, and how our new splitting criterion in IntCF improves on this.
Consider the data in Table~\ref{Tab:potential_outcomes-example}, with binary treatment $T$, potential outcomes $Y^{T=1}$ and $Y^{T=0}$, and two binary features $X_1, X_2$. For each sample, only the upright, bold outcome can be observed; the outcome in italics is unobservable. The treatment effect $Y^{T=1} - Y^{T=0}$ is zero for every sample. However, directly estimating the average treatment effect from the observed outcomes gives a biased estimate: the difference between the mean observed treated outcome $\frac{0.5+1.5+1.5}{3}=\frac{7}{6}$ and the mean observed control outcome $\frac{0.5+0.5+1.5}{3}=\frac{5}{6}$ is $\frac{1}{3} \neq 0$, because of the higher prevalence of treatment among samples with higher potential outcomes.
\begin{table}[tb]
  \centering
  \begin{tabular}{l|rrrrrr}
    \toprule
    Sample & 1 & 2 & 3 & 4 & 5 & 6 \\
    \midrule
    Treatment $T$ & 0 & 1 & 0 & 1 & 0 & 1 \\
    Treated outcome $Y^{T=1}$ & \unobserved{0.5} & \observed{0.5} & \unobserved{0.5} & \observed{1.5} & \unobserved{1.5} & \observed{1.5} \\
    Control outcome $Y^{T=0}$ & \observed{0.5} & \unobserved{0.5} & \observed{0.5} & \unobserved{1.5} & \observed{1.5} & \unobserved{1.5} \\
    Feature $X_1$ & 0 & 0 & 0 & 1 & 1 & 1 \\
    Feature $X_2$ & 0 & 0 & 0 & 0 & 1 & 1 \\
    \bottomrule
  \end{tabular}
  \caption{Example for data in a potential outcomes model, with binary treatment $T$, potential outcomes $Y^{T=1}$ and $Y^{T=0}$, and two binary features $X_1, X_2$. Note that for each sample only the outcomes corresponding to the treatment assignment can be observed (shown in bold). The outcomes in italics are unobserved and cannot be used for estimating the treatment effect.}\label{Tab:potential_outcomes-example}
\end{table}

Splitting the data into two subsets based on feature $X_1$ would eliminate this bias, leading to conditional average treatment effect estimates of 0 in both subsets. However, the splitting criterion used in causal forest \citep{wager_estimation_2018}, which is based on heterogeneity in outcomes, will not select this split, as the estimated outcome is the same on both sides. Even worse, the split on feature $X_2$, which results in an estimate of $\frac{1}{2}$ for the treatment effect of samples 1--4, an apparent effect that is itself an artifact of confounding rather than true heterogeneity, would seem preferable.

This is in contrast to the modification we propose in IntCF below. There, we explicitly combine two different splitting criteria, one for heterogeneity in the treatment effect and one for bias correction. With this approach, the change in predicted average treatment effect from splitting on $X_1$ is recognized as bias correction rather than heterogeneity, so the split is still used when building the tree. Interpretability is preserved, since the split is explicitly labeled as reducing bias rather than as improving the estimated treatment effect heterogeneity.

Applying IntCF to this example yields the decision stump shown in Figure~\ref{Fig:tree_stumps_example} (left) after the first split. While IntCF estimates the correct treatment effect in this example, causal forest produces a biased estimate, as the confounding introduced by feature $X_1$ is not removed (Figure~\ref{Fig:tree_stumps_example}, right). Our combined splitting criterion in IntCF correctly indicates that the selected split corrects bias rather than reflecting treatment effect heterogeneity.

\begin{figure}[tb]
  \centering
  \tikzset{sibling distance=35mm, level distance=25mm, align=center, every node/.style={anchor=north,rectangle}, every child node/.style={draw, minimum size=12mm}}
  \begin{tikzpicture}
    \node[draw,minimum size=20mm, text width=55mm] (intcf) {$X_1 \leq 0.5$\\ \textbf{Split criterion}: $\max(\estHet, \estBias)$\\ [2mm] Heterogeneity criterion: $\estHet=0$\\ Bias criterion\footnotemark: $\estBias=\frac{1}{9}$\\ [2mm] samples: 1, 2, 3, 4, 5, 6\\ estimate \( \hat\tau = \frac{1}{3} \)}
      child {node {samples: 1, 2, 3\\ estimate \( \hat\tau = 0 \)}}
      child {node {samples: 4, 5, 6\\ estimate \( \hat\tau = 0 \)}};
    \node [above=2mm of intcf] {IntCF};
  \end{tikzpicture}
  \hspace{1cm}
  \begin{tikzpicture}
    \node[draw,minimum size=20mm, text width=55mm] (causalforest) {$X_2 \leq 0.5$\\ \textbf{Split criterion:} $\widehat{\mathrm{Het}}=\frac{1}{18}$\\ \emph{no separate criteria \\ for heterogeneity \\ and bias} \\ samples: 1, 2, 3, 4, 5, 6\\ estimate \( \hat\tau = \frac{1}{3} \)}
      child {node {samples: 1, 2, 3, 4\\ estimate \( \hat\tau = \frac{1}{2} \)}}
      child {node {samples: 5, 6\\ estimate \( \hat\tau = 0 \)}};
    \node [above=2mm of causalforest] {Causal forest};
  \end{tikzpicture}
 \caption{Decision stumps after the first split for the data in Table~\ref{Tab:potential_outcomes-example}, using the combined splitting criterion of our proposed method, IntCF (left), and the criterion of the original causal forest algorithm \citep{wager_estimation_2018} (right). $\hat\tau$ denotes the estimated treatment effect in each node, computed as the difference between the mean observed outcome under treatment and under control. IntCF selects the split on $X_1$, correctly identifying it as bias correction rather than heterogeneity, and yields unbiased estimates in both child nodes; causal forest instead selects the split on $X_2$, which appears to indicate heterogeneity but remains biased due to the unaddressed confounding from $X_1$.}\label{Fig:tree_stumps_example}
\end{figure}
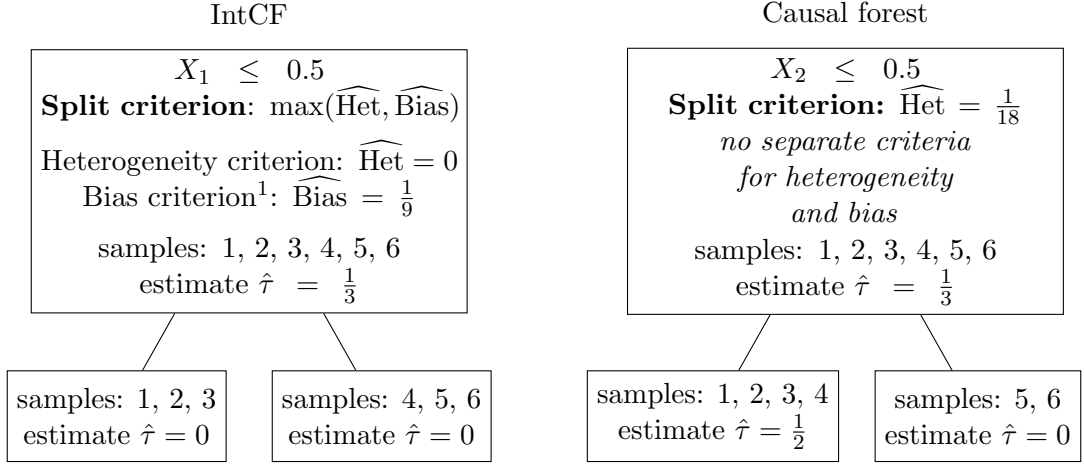
\footnotetext{Note that our bias criterion is an estimate for the squared bias.}% Footnote for Figure Fig:tree_stumps_example

\subsection{Outline of Paper}
The remainder of this paper is organized as follows. In Section~\ref{Sec:setting}, we introduce the considered setting, notation, and data model. In Section~\ref{Sec:algorithm}, we introduce the IntCF algorithm, with a particular focus on our new splitting criteria, as well as an additional step for honest predictions and validation of splits. In Section~\ref{Sec:theory}, we derive theoretical properties of our new splitting criteria under a change point model. In Section~\ref{Sec:simulations}, an extensive simulation study, including a semi-synthetic benchmark as well as a real data example, demonstrates that IntCF achieves competitive prediction accuracy while substantially improving interpretability. Finally, in Section~\ref{Sec:discussion}, we discuss our results and outline open questions.

\section{Preliminaries}\label{Sec:setting}
\subsection{Data and Potential Outcomes Model}\label{Sec:data_model}
We consider a potential outcomes framework \citep{rubin_estimating_1974}, in which the data generating process is described by the joint distribution of
\[
    (T, X, Y^{T=0}, Y^{T=1}),
\]
with binary treatment $T \in \{0,1\}$, a feature vector $X \in \R^d$, and potential outcomes $Y^{T=0}$ and $Y^{T=1}$. For each unit, only the potential outcome corresponding to the realized treatment is observed, i.e., we observe $Y \sim Y^{T}$, while the other potential outcome remains unobserved.
We aim to predict the \emph{conditional average treatment effect (CATE)}
\begin{equation}\label{Eq:tau}
    \tau(\features) := \expectation{Y^{T=1} - Y^{T=0} \mid \randomfeatures = \features},
\end{equation}
for a given feature vector \( \features \in \R^d \); throughout this paper, we refer to \( \tau(\features) \) as the individual treatment effect, since it is the treatment effect estimate available given the observed covariates \( \features \).
Beyond predicting \( \tau(\features) \), we are interested in interpreting treatment effect heterogeneity, i.e., in identifying which of the $d$ features actually influence \( \tau(\features) \).
We observe $N$ independent and identically distributed copies of \( (T, \randomfeatures, Y) \),
\[
    (t_i, \features_{i} = (x_{i, j})_{j=1}^{d}, y_i), \quad i = 1, \dotsc, N.
\]
The main difficulty in estimating $\tau(\features)$ is that, for each observation $i$, only one of $y_i^{T=1}$ and $y_i^{T=0}$ is observed; the individual treatment effect 
\[
    \tau_i := y_i^{T=1} - y_i^{T=0}
\]
can therefore never be observed, so standard machine learning methods for regression, which require direct observations of the target variable, cannot be applied to estimate $\tau$ directly.

We further define
\begin{equation}\label{Eq:propensity}
  p(\features) := P(T=1 \mid \randomfeatures = \features)
\end{equation}
for the \emph{treatment propensity}, and
\begin{equation}\label{Eq:main_effect}
  \mu(\features) := \frac{1}{2} \expectation{Y^{T=0} + Y^{T=1} \mid \randomfeatures = \features}
\end{equation}
for the \emph{main outcome}.

\subsection{Decision Trees and RFs}\label{Sec:intro-rf}
Decision trees for classification and regression tasks were proposed by \citet{breiman_classification_1984} and later extended to RFs \citep{breiman_random_2001}. RFs use a collection of decision trees, whose predictions are combined (for regression, usually by taking the average) to obtain an ensemble prediction. Each tree divides the feature space into subsets by performing binary splits based on a threshold applied to one of the features at each inner node. For deciding which split to use, a splitting criterion is applied, most commonly the impurity decrease according to an impurity measure, for example mean squared error (MSE) for regression.

Specifically, when selecting a split for a node \( \node_{\parent} \), which contains samples with indices in \( \indexTraining_{\parent} \), a collection of potential splits is considered; this might be all possible splits or a subset of these. Each potential split creates two child nodes, \( \node_{\leftchild} \) and \( \node_{\rightchild} \), containing samples with indices \( \indexTraining_{\leftchild} \) and \( \indexTraining_{\rightchild} \), respectively, such that \( \indexTraining_{\parent} = \indexTraining_{\leftchild} \sqcup \indexTraining_{\rightchild} \) is a disjoint union. We denote the number of samples in the left and right child nodes by \( n_{\leftchild} = \abs{\indexTraining_{\leftchild}} \) and \( n_{\rightchild} = \abs{\indexTraining_{\rightchild}} \), respectively. In a regression setting, each training sample has an attached response value \( y_i \). For each of these three nodes, a prediction and an impurity value are computed; comparing these then yields the impurity decrease for the candidate split. Finally, the potential split with the highest impurity decrease is selected as the actual split, and the procedure is repeated recursively at the two child nodes until a stopping criterion is reached.

\section{Interpretable Causal Tree and Forest Algorithm}\label{Sec:algorithm}

\subsection{Correcting Heterogeneity Splitting Rule for Confounding}\label{Sec:splitting-rule}
As outlined in Section~\ref{Sec:intro-rf}, a central element for the construction of decision trees from training data is the splitting rule. In this section, we will derive a new splitting rule, which can be applied to data from a potential outcomes model in a causal inference setting. Recall that the main difficulty compared to a regression setting is that for each observation $i$, only one of $y_i^{T=1}$ and $y_i^{T=0}$ is observed; the individual treatment effect 
$\tau_i = y_i^{T=1} - y_i^{T=0}$ can never be observed directly, hence, standard splitting rules as used for decision trees in regression, cannot be applied directly. Previous work has therefore considered different splitting criteria which overcome this problem \citep[see, e.g.,][]{athey_recursive_2016,wager_estimation_2018,athey_generalized_2019}. 

To motivate our new splitting criteria, as well as previously considered splitting criteria for causal forest \citep{wager_estimation_2018}, we start with a recap of the impurity decrease via MSE in the regression setting, the standard split statistic of RF implementations. There, the splitting criterion is obtained by maximizing among possible splits $\indexTraining_{\parent} = \indexTraining_{\leftchild} \sqcup \indexTraining_{\rightchild}$ the decrease in MSE, that is,
\begin{equation}\label{eq:mseRegression}
    \Delta^{\mathrm{RF}}(\leftchild, \rightchild) := \frac{1}{N} \sum_{i\in\indexTraining_{\parent}} (\hat{y}_{\parent} - y_{i})^{2} - \frac{1}{N} \left(\sum_{i\in\indexTraining_{\leftchild}} (\hat{y}_{\leftchild} - y_{i})^{2} + \sum_{i\in\indexTraining_{\rightchild}} (\hat{y}_{\rightchild} - y_{i})^{2}\right)
\end{equation}
with node predictions obtained as averages, that is, $\hat{y}_{\parent} := \frac{1}{n_\leftchild + n_\rightchild} \sum_{i\in\indexTraining_{\parent}} y_i$ and similar
\begin{equation}\label{eq:averagePredRegression}
  \hat{y}_{\leftchild} := \frac{1}{n_\leftchild} \sum_{i\in\indexTraining_{\leftchild}} y_i \quad \text{and} \quad
  \hat{y}_{\rightchild} := \frac{1}{n_\rightchild} \sum_{i\in\indexTraining_{\rightchild}} y_i.  
\end{equation}
Note that for such averages we have
\begin{equation}\label{eq:meanPredRegression}
    \hat{y}_{\parent} = \frac{n_\leftchild \cdot \hat{y}_{\leftchild} + n_\rightchild \cdot \hat{y}_{\rightchild}}{n_\leftchild + n_\rightchild}.
\end{equation}
Using~\eqref{eq:averagePredRegression}, it is easy to see that~\eqref{eq:mseRegression} is equal to
\begin{equation}\label{eq:mseRegression2}
 \frac{1}{N} \left(n_\leftchild (\hat{y}_\parent - \hat{y}_\leftchild)^2 + n_\rightchild (\hat{y}_\parent - \hat{y}_\rightchild)^2\right).
\end{equation}
Moreover, using~\eqref{eq:meanPredRegression}, it is easy to see that~\eqref{eq:mseRegression2} is equal to
\begin{equation}\label{eq:mseRegression3}
    \frac{n_\leftchild \cdot n_\rightchild}{N (n_\leftchild + n_\rightchild)} \left(\hat{y}_\leftchild - \hat{y}_\rightchild\right)^2.
\end{equation}
Note that~\eqref{eq:mseRegression3} equals, up to a factor of \( \frac{n_\leftchild + n_\rightchild}{N} \) (which is constant within each node), the variance among the predictions at the two child nodes. 

For the causal inference setting with outcome of interest being the (unobserved) treatment effect $\tau_i$, the direct analog to the splitting criterion from~\eqref{eq:mseRegression} is to consider
\begin{equation}\label{eq:mseCausal}
    \frac{1}{N} \sum_{i\in\indexTraining_{\parent}} (\hat{\tau}_{\parent} - \tau_{i})^{2} - \frac{1}{N} \left(\sum_{i\in\indexTraining_{\leftchild}} (\hat{\tau}_{\leftchild} - \tau_{i})^{2} + \sum_{i\in\indexTraining_{\rightchild}} (\hat{\tau}_{\rightchild} - \tau_{i})^{2}\right). 
\end{equation}
For the node predictions \( \hat{\tau}_{\parent}, \hat{\tau}_{\leftchild} \) and \( \hat{\tau}_{\rightchild} \) one can obtain natural estimates with the \emph{difference-in-means estimator} at a specific node $\node$, that is,
\begin{equation}\label{Eq:ate-estimate}
  \hat{\tau}_{\node} := \frac{1}{\abs{\{i \in \indexTraining_{\node}: t_{i} = 1\}}} \sum_{\substack{i \in \indexTraining_{\node} \\ t_{i}=1}} y_{i} - \frac{1}{\abs{\{i \in \indexTraining_{\node}: t_{i} = 0\}}} \sum_{\substack{i \in \indexTraining_{\node} \\ t_{i}=0}} y_{i},
\end{equation}
e.g., for the parent node $\node = \parent$, for the left child node $\node = \leftchild$, or for right child node $\node = \rightchild$.

The problem with the decrease in MSE as in~\eqref{eq:mseCausal} is, however, that the individual treatment effects $\tau_i$ are not observed and hence, one cannot use~\eqref{eq:mseCausal} for a splitting criterion.
Therefore, the proposal of \citet{wager_estimation_2018} for causal forest was to consider the analog of~\eqref{eq:mseRegression3} instead. That means, they propose to select splits $\indexTraining_{\parent} = \indexTraining_{\leftchild} \sqcup \indexTraining_{\rightchild}$ such that the variance among the predicted treatment effects at the child nodes is maximized, i.e., they maximize at each node
\begin{equation}\label{Eq:heterogeneity}
  \Delta^{\mathrm{CF}}(\leftchild, \rightchild) :=  \estHet := \frac{n_{\leftchild} \cdot n_{\rightchild}}{N (n_{\leftchild} + n_{\rightchild})} (\hat{\tau}_{\leftchild} - \hat{\tau}_{\rightchild})^{2}.
\end{equation}
We will call this the \emph{heterogeneity criterion}.
The major problem with using the heterogeneity criterion~\eqref{Eq:heterogeneity} as a replacement for the decrease in MSE in~\eqref{eq:mseCausal}, is that the corresponding analog equations to~\eqref{eq:averagePredRegression} and~\eqref{eq:meanPredRegression}, which were needed to derive the equivalence between~\eqref{eq:mseRegression} and~\eqref{eq:mseRegression3} in the regression setting, do not hold in the causal inference setting.
More precisely, the analog to~\eqref{eq:averagePredRegression} is given by the condition
\begin{equation}\label{eq:averagePredCausal}
    \hat{\tau}_{\leftchild} = \frac{1}{n_{\leftchild}} \sum_{i\in\indexTraining_{\leftchild}} \tau_{i} \quad \text{and} \quad \hat{\tau}_{\rightchild} = \frac{1}{n_{\rightchild}} \sum_{i\in\indexTraining_{\rightchild}} \tau_{i}
\end{equation}
and the analog to~\eqref{eq:meanPredRegression} is given by the condition
\begin{equation}\label{eq:meanPredCausal}
    \hat{\tau}_{\parent} = \frac{n_{\leftchild} \hat{\tau}_{\leftchild} + n_{\rightchild} \hat{\tau}_{\rightchild}}{n_{\leftchild} + n_{\rightchild}}.
\end{equation}
In contrast to the regression setting, \eqref{eq:averagePredCausal} and~\eqref{eq:meanPredCausal} are not guaranteed to hold in the causal inference setting. However, if these conditions do hold, one can derive equivalence between~\eqref{Eq:heterogeneity} and~\eqref{eq:mseCausal}, as the following two lemmas show.
\begin{lemma}\label{lem:mse_to_squared_change}
  If~\eqref{eq:averagePredCausal} holds, then the \emph{decrease in MSE}~\eqref{eq:mseCausal} equals the \emph{mean squared change in prediction}
\begin{equation}\label{Eq:mean-squared-change}
  \frac{1}{N} \left(n_{\leftchild} (\hat{\tau}_{\parent} - \hat{\tau}_{\leftchild})^{2} + n_{\rightchild} (\hat{\tau}_{\parent} - \hat{\tau}_{\rightchild})^{2}\right).
\end{equation}
\end{lemma}
The proof is given in the appendix. Note that~\eqref{Eq:mean-squared-change} is the analog to~\eqref{eq:mseRegression2} from the regression setting.
\begin{lemma}\label{lem:squared_change_to_het_criterion}
  If~\eqref{eq:meanPredCausal} holds, then the mean squared change in prediction~\eqref{Eq:mean-squared-change} equals the heterogeneity criterion $\estHet$ in~\eqref{Eq:heterogeneity}.
\end{lemma}
The proof is given in the appendix. 

Both~\eqref{eq:averagePredCausal} and~\eqref{eq:meanPredCausal} may fail to hold under confounding, and either violation breaks the equivalence between the heterogeneity criterion $\estHet$ in~\eqref{Eq:heterogeneity} and the decrease in MSE in~\eqref{eq:mseCausal}. However, the two assumptions differ fundamentally in nature. A violation of~\eqref{eq:meanPredCausal} is a property of the candidate split itself: for any given split, $\hat\tau_{\parent}$, $\hat\tau_{\leftchild}$, and $\hat\tau_{\rightchild}$ can be computed directly from the observed data, and one can check whether the parent estimate coincides with the sample-size-weighted average of the child estimates. Note that even in a randomized controlled trial, i.e., in the absence of any confounding, \eqref{eq:averagePredCausal} and~\eqref{eq:meanPredCausal} will typically not hold exactly, simply due to finite-sample noise. However, without confounding both conditions hold in expectation.
Under confounding, in contrast, they need not hold even in expectation. A strong violation of~\eqref{eq:meanPredCausal}, beyond what would be expected from sampling noise alone, therefore indicates a systematic effect rather than a chance fluctuation and, as illustrated in our example above, arises precisely when treatment propensity differs between the two child nodes, i.e., it is a direct symptom of confounding along the considered split. This makes it possible to explicitly measure and correct for such violations as part of the splitting criterion.

In contrast, \eqref{eq:averagePredCausal} is not a property of any particular split, but a standing identification assumption: it requires that, conditional on already being in a given node, the difference-in-means estimator~\eqref{Eq:ate-estimate} is unbiased for the true average treatment effect within that node, i.e., that no unobserved confounding remains once conditioned on the node. This within-node unconfoundedness assumption underlies the use of the difference-in-means estimator at any node regardless of the tree structure, and is required by essentially every tree-based method for treatment effect estimation, including our own leaf-level estimates. Since it is not tied to any single candidate split, it cannot be assessed or corrected by comparing $\hat\tau_{\parent}$, $\hat\tau_{\leftchild}$, and $\hat\tau_{\rightchild}$ at a given step, but only by conditioning on the relevant confounders while growing the tree as a whole. We therefore focus our new splitting criterion on correcting for violations of~\eqref{eq:meanPredCausal}, while retaining~\eqref{eq:averagePredCausal} as a standing assumption, consistent with its role throughout the tree-based causal inference literature.

As Lemma~\ref{lem:squared_change_to_het_criterion} shows, whenever the mean squared change in prediction~\eqref{Eq:mean-squared-change} and the heterogeneity criterion $\estHet$ in~\eqref{Eq:heterogeneity} are not equal, this can be directly attributed to a violation of~\eqref{eq:meanPredCausal} and hence, a bias correction. We therefore treat the difference between~\eqref{Eq:mean-squared-change} and~\eqref{Eq:heterogeneity} as one component of our splitting criterion, quantifying the extent to which a split corrects for confounding, as opposed to detecting heterogeneity in the treatment effect. The following Lemma gives an explicit expression how this difference corresponds to the violation of~\eqref{eq:meanPredCausal}.
\begin{lemma}\label{lem:derivation_bias_criterion}
  The difference of~\eqref{Eq:mean-squared-change} and~\eqref{Eq:heterogeneity} is given by
\begin{equation}\label{Eq:bias}
  \estBias := \frac{n_{\leftchild} + n_{\rightchild}}{N} \left(\hat{\tau}_{\parent} - \frac{n_{\leftchild} \hat{\tau}_{\leftchild} + n_{\rightchild} \hat{\tau}_{\rightchild}}{n_{\leftchild}+n_{\rightchild}}\right)^{2}.
\end{equation}
\end{lemma}
We call~\eqref{Eq:bias} the \emph{bias criterion}.
The proof is given in the appendix. 

In total, we have shown the following decomposition
\begin{equation}\label{Eq:biasVarianceDecomp}
    \text{MSE decrease} \approx \frac{1}{N} \left(n_{\leftchild} (\hat{\tau}_{\parent} - \hat{\tau}_{\leftchild})^{2} + n_{\rightchild} (\hat{\tau}_{\parent} - \hat{\tau}_{\rightchild})^{2}\right) = 
\estHet + \estBias.
\end{equation}
Note that the decomposition in~\eqref{Eq:biasVarianceDecomp} is analog to a classical bias-variance decomposition. The two criteria, $\estHet$ and $\estBias$, capture different reasons for improved predictions: identified heterogeneity (which we measure by the variance of predictions) and bias correction (which we estimate by the square of change of average prediction).

To improve interpretability of the final trees, we want to separate these two reasons. Therefore, at each node, we select the split that has the strongest effect from either of these two sources. As both the heterogeneity criterion $\estHet$ in~\eqref{Eq:heterogeneity} and the bias criterion $\estBias$ in~\eqref{Eq:bias} are directly derived from the mean squared change~\eqref{Eq:mean-squared-change}, their values are comparable. Hence, with IntCF we propose to \emph{use the maximum of both criteria to select the split}. That is, our splitting criterion for IntCF is obtained by maximizing among possible splits $\indexTraining_{\parent} = \indexTraining_{\leftchild} \sqcup \indexTraining_{\rightchild}$ the maximum of heterogeneity and bias improvement, that is,
\begin{equation}\label{Eq:IntCF}
    \Delta^{\mathrm{IntCF}}(\leftchild, \rightchild) := \max\left(\estHet, \estBias\right).
\end{equation}

For our theoretical analysis, we also consider signed versions of these criteria, which we refer to as signed splitting criteria, defined by
\begin{equation}\label{Eq:signed-het}
  \estSignedHet := \frac{\sqrt{n_{\leftchild} \cdot n_{\rightchild}}}{n_{\leftchild} + n_{\rightchild}} (\hat{\tau}_{\leftchild} - \hat{\tau}_{\rightchild}),
\end{equation}
the \emph{signed heterogeneity}, and
\begin{equation}\label{Eq:signed-bias}
  \estSignedBias := \hat{\tau}_{\parent} - \frac{n_{\leftchild} \hat{\tau}_{\leftchild} + n_{\rightchild} \hat{\tau}_{\rightchild}}{n_{\leftchild}+n_{\rightchild}},
\end{equation}
the \emph{signed bias}. By squaring and then multiplying by $\frac{n_{\leftchild} + n_{\rightchild}}{N}$ (a constant factor at each node), one recovers the heterogeneity criterion $\estHet$ from the signed heterogeneity $\estSignedHet$, and the bias criterion $\estBias$ from the signed bias $\estSignedBias$.

\subsection{Honest Validation of Splits}\label{Sec:validation-of-splits}
In this section, we introduce a second modification to standard RF that we propose for IntCF, complementing the new splitting criterion~\eqref{Eq:IntCF} introduced in Section~\ref{Sec:splitting-rule}. Importantly, this modification does not affect how the trees themselves are grown: trees are constructed exactly as before, simply using the combined splitting criterion~\eqref{Eq:IntCF} in place of the standard impurity decrease. Instead, this modification concerns how we later summarize, at each node, the heterogeneity and bias contribution identified during tree construction, for example when computing a mean-decrease-in-impurity-type measure of feature importance.

While building the tree, especially for nodes close to the root, the criteria used to select splits might not yet be informative: as discussed in Section~\ref{Sec:splitting-rule}, the within-node unconfoundedness assumption~\eqref{eq:averagePredCausal} need not hold until the tree has conditioned on the relevant confounders, so for nodes near the root the estimates entering the heterogeneity criterion $\estHet$~\eqref{Eq:heterogeneity} and the bias criterion $\estBias$ \eqref{Eq:bias} may still be affected by confounding that is only corrected at later splits. 

To address this problem, we use out-of-bag (or hold-out) samples to summarize the heterogeneity and bias contribution of each split. This overall idea is not new and has been used in various forms before, both for prediction, where it is known as \emph{honest prediction}, and for feature importance, e.g., in debiased or out-of-bag variants of MDI \citep{li_debiased_2019,zhou_unbiased_2021}. We build on both of these ideas. First, we also use out-of-bag or hold-out samples for our final predictions; this is simply the standard honest-prediction procedure of double-sample trees, as used in causal forest \citep{wager_estimation_2018}, and not itself a new contribution. Second, and in addition, we use these samples as validation set to re-evaluate the split statistics at every node, similar to out-of-bag MDI approaches for RF. However, simply recomputing the heterogeneity criterion $\estHet$~\eqref{Eq:heterogeneity} and the bias criterion $\estBias$ \eqref{Eq:bias} on out-of-bag samples, using the same difference-in-means estimator at each node, is not sufficient in the causal inference setting: while using held-out data addresses the classical overfitting bias that motivates honest estimation and out-of-bag MDI, it does not address confounding. In particular, the within-node unconfoundedness assumption~\eqref{eq:averagePredCausal} need not hold until the tree has conditioned on the relevant confounders, regardless of whether the underlying estimate is computed on training or on out-of-bag data. We therefore develop a modified validation procedure in the remainder of this section, which draws on the fully grown tree (or forest) to correct for this remaining confounding.

To use the validation set to also validate split attribution at all nodes, we extend the double-sample-trees procedure to compute predictions from this data set not only at leaf nodes, but also at inner nodes. To obtain both an estimate of corrected bias and of identified heterogeneity, we compute two separate estimates for each node, using the following procedure.

For the first estimate, we feed the validation samples through the considered tree and, at each node, compute an estimate as before using the difference-in-means estimator~\eqref{Eq:ate-estimate}, but now based on validation samples instead of training samples. Denote the indices of validation samples in node \( \node \) by \( \indexValidation_{\node} \); throughout this section, $n_{\node}$, $n_{\leftchild}$, $n_{\rightchild}$, and $N$ refer to the corresponding counts of validation samples, analogous to their earlier use for training samples in Section~\ref{Sec:splitting-rule}. At node \( \node \), we call this estimate
\begin{equation}\label{Eq:forwardPred}
    \hat{\tau}_{\node}^{\forward} = \frac{1}{\abs{\{i \in \indexValidation_{\node}: t_{i} = 1\}}} \sum_{\substack{i \in \indexValidation_{\node} \\ t_{i}=1}} y_{i} - \frac{1}{\abs{\{i \in \indexValidation_{\node}: t_{i} = 0\}}} \sum_{\substack{i \in \indexValidation_{\node} \\ t_{i}=0}} y_{i} 
\end{equation}
the \emph{forward prediction}. It might still be influenced by bias that is only corrected at later splits. As with double-sample trees, at leaf nodes these estimates are used as predictions for samples falling in the respective leaf.

For the second estimate, we use the final predictions \( \hat{\tau}_{i}^{\validaton} \) for the validation samples (these can be either the predictions from the single tree or from an entire RF). At each node \( \node \), a new estimate \( \hat{\tau}_{\node}^{\backward} \) for the CATE can then be computed by averaging the final predictions of validation samples contained in this node, i.e., 
\begin{equation}\label{Eq:backwardPred}
    \hat{\tau}_{\node}^{\backward} = \frac{1}{\abs{\indexValidation_{\node}}} \sum_{i \in \indexValidation_{\node}} \hat{\tau}_{i}^{\validaton}.
\end{equation}
For inner nodes, this estimate can also be computed as \( \hat{\tau}_{\parent}^{\backward} = \frac{n_{\leftchild} \hat{\tau}_{\leftchild}^{\backward} + n_{\rightchild} \hat{\tau}_{\rightchild}^{\backward}}{n_{\leftchild} + n_{\rightchild}} \). Since these estimates can therefore be computed iteratively from the leaves up to the root, we call them \emph{backward predictions}. If the final predictions are unbiased, so are the backward predictions.

One task in validating split attribution is to estimate the heterogeneity between treatment effects in the child nodes at each split. For this, we use a formula analogous to the heterogeneity splitting criterion~\eqref{Eq:heterogeneity}, but now using the backward predictions, i.e.,
\begin{equation}\label{Eq:validation-het}
  \estHetV := \frac{n_{\leftchild} \cdot n_{\rightchild}}{N (n_{\leftchild} + n_{\rightchild})} \left(\hat{\tau}_{\leftchild}^{\backward} - \hat{\tau}_{\rightchild}^{\backward}\right)^{2}.
\end{equation}
We refer to~\eqref{Eq:validation-het} as the \emph{heterogeneity validation criterion}.

For the bias, such a direct transfer is not possible: unlike the heterogeneity criterion, the backward predictions already incorporate bias corrections made throughout the tree, not only the correction attributable to the split under consideration. Instead, we can estimate the squared bias corrected by the remainder of the tree by comparing forward and backward estimates, using \( \frac{n_{\leftchild} + n_{\rightchild}}{N} \left(\hat{\tau}_{\parent}^{\forward} - \hat{\tau}_{\parent}^{\backward}\right)^{2} \). But this also includes bias that was only corrected in descendant nodes. For that reason, we instead use the decrease in squared bias,
\begin{equation}\label{Eq:validation-bias}
  \estBiasV := \frac{n_{\leftchild} + n_{\rightchild}}{N} \left(\hat{\tau}_{\parent}^{\forward} - \hat{\tau}_{\parent}^{\backward}\right)^{2} - \frac{n_{\leftchild}}{N} \left(\hat{\tau}_{\leftchild}^{\forward} - \hat{\tau}_{\leftchild}^{\backward}\right)^{2} - \frac{n_{\rightchild}}{N} \left(\hat{\tau}_{\rightchild}^{\forward} - \hat{\tau}_{\rightchild}^{\backward}\right)^{2},
\end{equation}
which we refer to as the \emph{bias validation criterion}.

In each node, one can show that sum of the heterogeneity validation criterion and the bias validation criterion equals the decrease in mean squared error obtained by replacing the forward estimate of the average treatment effect with the final individual predictions, as the following lemma shows.
\begin{lemma}\label{lem:sum_of_validation}
 For the validation criteria~\eqref{Eq:validation-het} and~\eqref{Eq:validation-bias} we have
  \[
      \estHetV + \estBiasV = \frac{1}{N} \left( 
      \sum_{i\in\indexValidation_{\parent}} \left(\hat{\tau}_{\parent}^{\forward} - \hat{\tau}_{i}^{\validaton}\right)^{2} - \sum_{i\in\indexValidation_{\leftchild}} \left(\hat{\tau}_{\leftchild}^{\forward} - \hat{\tau}_{i}^{\validaton}\right)^{2} - \sum_{i\in\indexValidation_{\rightchild}} \left(\hat{\tau}_{\rightchild}^{\forward} - \hat{\tau}_{i}^{\validaton}\right)^{2}\right).
  \]
\end{lemma}
The proof is given in the appendix. Again, note that the decomposition in Lemma~\ref{lem:sum_of_validation} is analog to a classical bias-variance decomposition, similar to~\eqref{Eq:biasVarianceDecomp} for the splitting criteria.

\subsection{Summary of Algorithm}
We now combine the two modifications introduced above---the amended splitting criterion $\Delta^{\mathrm{IntCF}}$~\eqref{Eq:IntCF} from Section~\ref{Sec:splitting-rule} and the honest validation procedure from Section~\ref{Sec:validation-of-splits}---into a complete algorithm for growing decision trees and RFs. Analogous to how a RF is built from individual decision trees (Section~\ref{Sec:intro-rf}), we first describe how to grow a single tree, which we call the \emph{interpretable causal tree (IntCT)}; the corresponding forest, \emph{IntCF}, is then obtained by combining many such trees, as described below.

\subsubsection{Interpretable Causal Tree (IntCT)}

\begin{enumerate}
  \item Split the samples into a training set \( \indexTraining \) and a validation set \( \indexValidation \).
  \item Build the tree based on the training samples \( \indexTraining \):
  \begin{enumerate}
    \item At the current node, consider all possible splits, determine the resulting sample sets for the two child nodes, and compute the treatment-effect estimates $\hat{\tau}_{\leftchild}, \hat{\tau}_{\rightchild}$ for these child nodes using the difference-in-means estimator~\eqref{Eq:ate-estimate}.
    \item For each possible split, compute the heterogeneity criterion $\estHet$~\eqref{Eq:heterogeneity} and the bias criterion $\estBias$~\eqref{Eq:bias}.
    \item Select the split that maximizes $\Delta^{\mathrm{IntCF}} = \max\left(\estHet, \estBias\right)$~\eqref{Eq:IntCF}, i.e., the split with the largest value of either criterion.
    \item Create the new child nodes and assign them their respective training samples.
    \item If no stopping criterion (e.g., maximum depth or minimum leaf size) has been reached, repeat the above steps at each child node.
  \end{enumerate}
  \item Validate the splits and compute predictions using the validation samples (see Section~\ref{Sec:validation-of-splits}); this step retraces the tree built in Step~2, but now operates on $\indexValidation$ instead of $\indexTraining$:
  \begin{enumerate}
    \item Compute the forward predictions $\hat{\tau}_{\node}^{\forward}$~\eqref{Eq:forwardPred} for all nodes of the tree.
    \item At the leaves, use the forward prediction as the tree's prediction for samples that fall into the respective leaf; this is the standard ``honest prediction'' procedure, as proposed, e.g., in \citep{wager_estimation_2018}.
    \item Use the tree (or the forest) to compute estimated treatment effects $\hat{\tau}_{i}^{\validaton}$ for all validation samples.
    \item Use these estimated treatment effects to compute the backward predictions $\hat{\tau}_{\node}^{\backward}$ according to~\eqref{Eq:backwardPred} for all nodes of the tree.
    \item Compute the heterogeneity validation criterion $\estHetV$~\eqref{Eq:validation-het} and the bias validation criterion $\estBiasV$~\eqref{Eq:validation-bias} from the forward and backward predictions, and store their values at each node.
  \end{enumerate}
  \item The final result is the tree together with its leaf-level predictions and the validation criterion values $\estHetV, \estBiasV$ stored at each split.
\end{enumerate}

\subsubsection{Interpretable Causal Forest (IntCF)}

In a forest, the same algorithm is used to grow each tree, with the training/validation split of Step~1 performed independently for each tree, so that every sample is used, either in the training or in the validation role, in every tree of the forest. This plays a role analogous to the bootstrap sampling used in standard RFs, where the samples not drawn in the bootstrap, the out-of-bag samples, serve as the validation set. As with standard RFs, only a random subset of features is considered at each split when growing each tree.

\subsubsection{Feature Importance}

Following the original mean-decrease-in-impurity approach to feature importance \citep{breiman_classification_1984}, we sum, over all nodes that split on a given feature, a measure of that split's contribution to obtain the feature's importance. Instead of the impurity decrease, we use a validation criterion for this purpose: by default, only the heterogeneity validation criterion $\estHetV$~\eqref{Eq:validation-het} is used, but the bias validation criterion $\estBiasV$~\eqref{Eq:validation-bias}, or the sum of both, can also be used if desired.

In many implementations, for example scikit-learn \citep{scikit-learn}, it is customary to scale the importance measure so that it sums to 1. In contrast to this standard MDI procedure, the bias validation criterion can take negative values, so feature importance scores that include it may also be negative. We therefore scale our feature importance scores so that the sum of the non-negative scores is 1.

\section{Population-level Analysis of the Splitting Criteria under a Change Point Model}\label{Sec:theory}

In this section, we provide theoretical support for the central claim of this paper: that the combined splitting criterion introduced in Section~\ref{Sec:splitting-rule} correctly distinguishes splits driven by confounding-induced bias from splits that reflect genuine heterogeneity in the treatment effect. Rather than finite-sample guarantees, we study this question at the population level, i.e., in the limit as the sample size $N \to \infty$, under a stylized \emph{change point model} in which the treatment effect, the propensity, and the main outcome are step functions of a single covariate.

The remainder of this section is organized as follows. Section~\ref{Sec:pop-level-analysis} derives the population limits of the signed heterogeneity criterion $\estSignedHet$ and the signed bias criterion $\estSignedBias$. Section~\ref{Sec:marginalization} explains why, due to the marginalization inherent to decision-tree splits, it suffices to study a univariate change point model, which we introduce formally. Section~\ref{Sec:decomposition} then shows that any such change point model decomposes into a component with a constant treatment effect and a component for which the difference-in-means estimator is unbiased. Finally, Section~\ref{Sec:identifying-het-bias} shows that our two splitting criteria correctly recover this decomposition: the population bias criterion attains an extremum at the true change point exactly when bias is present, and the population heterogeneity criterion attains an extremum there exactly when the treatment effect is genuinely heterogeneous. This provides theoretical justification for the algorithm's ability to separate these two sources of improvement, which underlies its interpretability.

\subsection{Population-level Analysis}\label{Sec:pop-level-analysis}

Rather than deriving finite-sample results, we focus on the corresponding population analogs of~\eqref{Eq:signed-het} and~\eqref{Eq:signed-bias} in order to elucidate general structural properties of these splitting criteria.
To this end, consider a decision tree with an inner node $\node$ and the associated hyper-rectangle $R(\node)$.
At node $\node$, we examine a candidate split along a given variable, say $x_1$, at threshold $k$.
Throughout the analysis, we impose the following simplifying assumptions:
\begin{enumerate}
    \item[A1] The feature vector $X$ is uniformly distributed on the unit hypercube, that is, $X \sim U([0,1]^d)$.
    \item[A2] The outcomes $Y^{T = 1}, Y^{T = 0}$ are bounded.
    \item[A3] The positivity assumption holds, meaning that there exists some $\epsilon > 0$ such that, for all $x \in [0,1]^d$, 
    \[
        \epsilon \leq P(T = 1 \mid X = x) \leq 1 - \epsilon .
    \]
    \item[A4] The hyper-rectangle $R(\node)$, with positive volume $\mu(R(\node)) > 0$, the choice of the splitting variable (without loss of generality assumed to be $x_1$), as well as the considered threshold $k$ are independent of the training data $D = \{(y_1,\features_1,t_1),\dotsc,(y_N,\features_N,t_N)\}$.
\end{enumerate}

Note that, since decision trees are invariant under monotone transformations of the covariates, the first assumption is essentially equivalent to assuming independence among the different features.
Assumptions~A1 and~A2 are standard in the analysis of decision-tree-based methods \citep[see, for example,][]{behr_provable_2022,wager_estimation_2018}. Assumption~A3 is a requirement for causal inference in general \citep[compare, for example,][]{rosenbaum_propensity_1983}.
The final assumption is clearly violated for our intCT algorithm, since all split points in the tree are selected in a data-dependent manner, implying that the resulting hyper-rectangle $R(\node)$ likewise depends on the training data $D$.
Nevertheless, extending the subsequent results to this setting appears to be relatively straightforward and mainly of a technical nature; see Remark~\ref{rem:uniform} for details.

Under Assumptions~A1--A4, it follows directly from the law of large numbers and the continuous mapping theorem that, when the training data $D$ are generated i.i.d.\ according to the data-generating process $P(Y^{T=0}, Y^{T=1}, X, T)$, for any fixed threshold \( k \)
\begin{equation}\label{eq:convsHsB}
     \estSignedHet \pto \signedHet(k),
     \qquad
     \estSignedBias \pto \signedBias(k),
     \qquad \text{as } N \to \infty .
\end{equation}
The corresponding population quantities are given by
\begin{equation}\label{eq:popsH}
  \signedHet(k) := \sqrt{k (1-k)} \cdot \left[
  \bigl(
  \bar{Y}^{T=1}_{\leftchild} - \bar{Y}^{T=0}_{\leftchild}
  \bigr) 
  -
  \bigl(
  \bar{Y}^{T=1}_{\rightchild} - \bar{Y}^{T=0}_{\rightchild}
  \bigr) \right],
\end{equation}
and
\begin{equation}\label{eq:popsB}
    \signedBias(k) :=
    \bigl(\bar{Y}^{T=1}
    -
    \bar{Y}^{T=0}
    \bigr)
    -
    k \cdot
    \bigl(
    \bar{Y}^{T=1}_{\leftchild}
    -
    \bar{Y}^{T=0}_{\leftchild}
    \bigr)
    -
    (1-k) \cdot
    \bigl(\bar{Y}^{T=1}_{\rightchild}
    - \bar{Y}^{T=0}_{\rightchild}
    \bigr)
\end{equation}
where, for \( a \in \{0, 1\} \),
\begin{align*}
  \bar{Y}^{T=a} &:= \expectation{Y^{T=a} \mid X \in R(\node),\, T=a},\\
  \bar{Y}^{T=a}_{\leftchild} &:= \expectation{Y^{T=a} \mid X_1 \leq k,\, X \in R(\node),\, T=a},\\
  \bar{Y}^{T=a}_{\rightchild} &:= \expectation{Y^{T=a} \mid X_1 > k,\, X \in R(\node),\, T=a}
\end{align*}
are expectations of observed outcomes in the whole node as well as to left and right of the considered split, respectively.
In the following, we study the properties of the population counterparts $\signedHet(k)$ and $\signedBias(k)$ associated with our proposed splitting criteria.

\begin{remark}\label{rem:uniform}
    To dispense with Assumption~4, it would suffice to establish the convergence result in~\eqref{eq:convsHsB} uniformly over all hyper-rectangles $R(\node)$ with a fixed minimal volume.
More precisely, one would require a result of the form that, for any $\epsilon > 0$,
\[
    \probability{\sup_{R,k,j\in \{1,\dotsc,d\},\, \mu(R) > \delta}
    \bigl| \estSignedHet - \signedHet(k) \bigr|
    > \epsilon} \;\to\; 0
    \quad \text{as } N \to \infty,
\]
and an analogous statement for the bias term.
The assumption of a minimal volume can be enforced in practice by imposing a maximal tree depth and by requiring splitting thresholds $k$ to be bounded away from zero and one; both restrictions are standard in the analysis of RFs \citep[see, e.g.,][]{scornet_theory_2026}.
Establishing such uniform convergence results should be feasible using standard tools from empirical process theory, including uniform convergence arguments and concentration inequalities \citep[see, e.g.,][]{behr_provable_2022}.
Since this extension is largely technical and does not affect the qualitative insights of our analysis, we do not pursue it further here and instead focus on the population quantities $\signedHet(k)$ and $\signedBias(k)$ in~\eqref{eq:popsH} and~\eqref{eq:popsB} directly.
\end{remark}

\subsection{Marginalization Effects}\label{Sec:marginalization}

By construction, when a decision tree considers a split along a given variable $X_j$, say $X_1$, the influence of the remaining variables $X_2,\dotsc,X_d$ is marginalized out through averaging, as reflected in the conditional expectations in \eqref{eq:popsH} and \eqref{eq:popsB}. Consequently, once a specific variable $X_1$ is selected, decision trees can capture only the marginal effect of this variable. This is an intrinsic property of decision trees and gives rise to well-known limitations of such methods.

This marginalization is visible directly in the population criteria themselves: $\signedHet(k)$ and $\signedBias(k)$ in \eqref{eq:popsH} and \eqref{eq:popsB} depend on the joint distribution of $(\randomfeatures,T,Y^{T=0},Y^{T=1})$ only through $X_1$ and the conditioning event $\randomfeatures \in R(t)$, which is fixed across both candidate children. It therefore suffices, for the purpose of studying the behavior of a single candidate split, to consider a data-generating process in which only one covariate, say $x_1$, carries any signal. Here, we study a univariate \emph{change point model} for CATE, propensity, and main outcome, in analogy to the general definitions of $\tau, p, \mu$ in \eqref{Eq:tau},~\eqref{Eq:propensity} and~\eqref{Eq:main_effect}:
\begin{equation}\label{eq:change_point_model}
\begin{aligned}
    \tau(x_1) &:= \beta_0 + \beta_1 \indicator{x_1 \leq \gamma}, \\
    p(x_1) &:= q_0 + q_1 \indicator{x_1 \leq \gamma},\\
    \mu(x_1) &:= \alpha_0 + \alpha_1 \indicator{x_1 \leq \gamma}. 
\end{aligned}
\end{equation}

Each coefficient governs a distinct aspect of the model. The slope $\beta_1$ governs treatment effect heterogeneity: $\beta_1 = 0$ corresponds to a constant treatment effect $\tau(x_1)=\beta_0$, i.e., a model without treatment effect heterogeneity. Analogously, $q_1$ governs the degree of confounding through $x_1$: $q_1=0$ corresponds to a constant propensity $p(x_1)=q_0$, i.e., a setting in which treatment assignment does not depend on $x_1$, as in a randomized trial. Finally, $\alpha_1$ governs whether $x_1$ has a direct, prognostic effect on the main outcome: $\alpha_1=0$ corresponds to a constant main outcome $\mu(x_1)=\alpha_0$. Intuitively, confounding bias for the difference-in-means estimator arises when $x_1$ affects both treatment assignment ($q_1\neq0$) and the potential outcomes; Theorem~\ref{lem:decomposition} below makes this precise.

By construction, this model has a unique natural split point at $\gamma$: splitting exactly there separates $\tau$, $p$, and $\mu$ into constant pieces on either side, which is what allows the resulting heterogeneity and bias to be exactly quantified.

With $\mu$ and $\tau$ specified by the change point model above---now understood as depending on $\features$ only through its first coordinate $x_1$---the potential outcomes are defined as
\begin{equation}\label{Eq:outcomes_from_model}
  Y^{T=0}(\features) = \mu(\features) - \frac{1}{2} \tau(\features) + \noise_0 \text{ and } Y^{T=1}(\features) = \mu(\features) + \frac{1}{2} \tau(\features) + \noise_1
\end{equation}
with mean-zero noise terms $\noise_0, \noise_1$ that are assumed to be independent of one another and of $\randomfeatures$.

\subsection{Bias--Heterogeneity Decomposition of the Change Point Model}\label{Sec:decomposition}

In the following, we provide theoretical insight into why the population splitting criteria $\signedHet(k)$ and $\signedBias(k)$ in~\eqref{eq:popsH} and~\eqref{eq:popsB} are able to distinguish between heterogeneity in the treatment effect and bias induced by confounding.
To this end, we first observe that any change point model of the form~\eqref{eq:change_point_model} can be decomposed into the sum of two components: one corresponding to a model with a constant treatment effect, and another for which the difference-in-means estimator is unbiased.
This decomposition is formalized in the following lemma.
More precisely, consider a potential outcome model $(Y^{T = 0}, Y^{T = 1}, X, T)$ with treatment effect $\tau(x)$, main outcome $\mu(x)$, and propensity $p(x)$ as defined in Section~\ref{Sec:data_model}.
We say that the \emph{difference-in-means estimator is unbiased} if
\begin{equation}\label{eq:unbiased}
    \expectation{\tau(X)} = \expectation{Y^{T = 1} \mid T = 1} - \expectation{Y^{T = 0} \mid T = 0},
\end{equation}
or, equivalently (recall Equation~\ref{Eq:outcomes_from_model}), if
\[
    \expectation{\tau(X)} =
    \expectation{\mu(X) + \tfrac{1}{2}\tau(X) \mid T = 1 }
    - \expectation{\mu(X) - \tfrac{1}{2}\tau(X) \mid T = 0 }.
\]
\begin{theorem}\label{lem:decomposition}
Consider a potential outcome model $(Y^{T = 0}, Y^{T = 1}, X, T)$ with a single uniformly distributed covariate \( X \sim U([a, b]), a < b \in \R \), potential outcomes as in~\eqref{Eq:outcomes_from_model}, and main outcome, treatment effect, and propensity following a change point model as in~\eqref{eq:change_point_model}, i.e.,
\[
\mu(x) = \alpha_{0} + \alpha_{1} \cdot \indicator{x \leq \gamma},\quad
\tau(x) = \beta_{0} + \beta_{1} \cdot \indicator{x \leq \gamma},\quad
p(x) = q_{0} + q_{1} \cdot \indicator{x \leq \gamma}.
\]
Assume that Assumption~A3 holds. Then $\mu(x)$ and $\tau(x)$ can be decomposed as
\[
\mu(x) = \mu^B(x) + \mu^H(x)
\quad \text{and} \quad
\tau(x) = \tau^B(x) + \tau^H(x),
\]
such that, for suitable constants $\alpha_{0}^\prime, \alpha_{1}^\prime, \alpha_{0}^{\prime\prime}, \alpha_{1}^{\prime\prime}, \beta_{0}^\prime, \beta_{0}^{\prime\prime}, \beta_{1}^{\prime\prime}$, the following holds:
\begin{enumerate}
    \item \textbf{(Change point model with constant treatment effect)}\\
    $\mu^B(x) = \alpha_{0}^\prime + \alpha_{1}^\prime \cdot \indicator{x \leq \gamma}$ and
    $\tau^B(x) = \beta_{0}^\prime$.
    
    \item \textbf{(Change point model with unbiased difference-in-means)}\\
    $\mu^H(x) = \alpha_{0}^{\prime\prime} + \alpha_{1}^{\prime\prime} \cdot \indicator{x \leq \gamma}$ and
    $\tau^H(x) = \beta_{0}^{\prime\prime} + \beta_{1}^{\prime\prime} \cdot \indicator{x \leq \gamma}$,
    where
    \begin{equation}\label{eq:noBias}
        \expectation{\tau^H(X)} =
        \expectation{\mu^H(X) + \tfrac{1}{2}\tau^H(X) \mid T = 1 }
        -
        \expectation{\mu^H(X) - \tfrac{1}{2}\tau^H(X) \mid T = 0 }.
    \end{equation}
\end{enumerate}
\end{theorem}
The proof is given in the appendix.
\begin{remark}
  Note that this decomposition is not unique. The condition~\eqref{eq:noBias} does not restrict \( \alpha_{0}^{\prime} \) and \( \alpha_{0}^{\prime\prime} \) as well as \( \beta_{0}^{\prime} \) and \( \beta_{0}^{\prime\prime} \), so that \( \alpha_{0} \) and \( \beta_{0} \) can be freely distributed onto the respective pair of parameters. If the treatment assignment is randomized, i.e., \( p(x) \) constant in \( [0, 1] \), any parameters will fulfill~\eqref{eq:noBias}, so the same freedom of choice also extends to \( \alpha_1 \).
\end{remark}

Intuitively, this decomposition isolates exactly the two reasons a split can improve the estimated treatment effect (cf.\ Section~\ref{Sec:marginalization}): the bias component $(\mu^B,\tau^B)$ has, by construction, a constant treatment effect, so any change in the estimated treatment effect attributable to this component must be due to bias correction rather than genuine heterogeneity. Conversely, the heterogeneity component $(\mu^H,\tau^H)$ satisfies the unbiasedness condition~\eqref{eq:noBias} by construction, so any change in its estimated treatment effect must reflect genuine heterogeneity rather than bias correction. We make this precise in Section~\ref{Sec:identifying-het-bias} below.
Figure~\ref{Fig:model_decomposition} illustrates such a decomposition as described in Theorem~\ref{lem:decomposition}.

\begin{figure}[tb]
  \centering
  \includegraphics[height=0.8\textwidth]{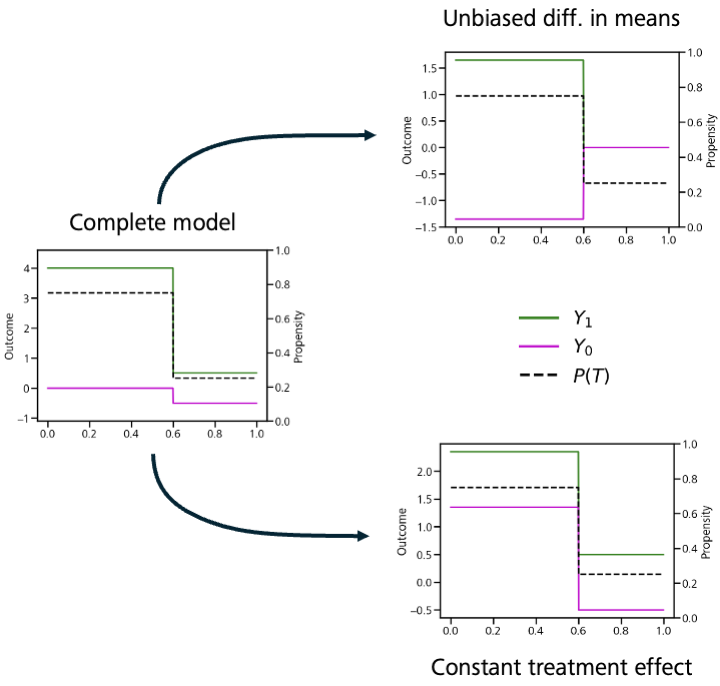}
  \caption{Illustration of the decomposition of a change point model into a component with constant treatment effect and a component with unbiased difference-in-means, as in Theorem~\ref{lem:decomposition}. The original model was defined by \( \mu(x)=0+2\cdot\indicator{x\leq0.6}, \tau(x)=1+3\cdot\indicator{x\leq0.6} \) and \( p(x)=0.25+0.5\cdot\indicator{x\leq0.6} \).}
  \label{Fig:model_decomposition}
\end{figure}

\subsection{Identifying Heterogeneity and Bias with Signed Splitting Criteria}\label{Sec:identifying-het-bias}

By Theorem~\ref{lem:decomposition} and the linearity of expectation, it follows directly that the population splitting criteria $\signedHet(k)$ and $\signedBias(k)$ in~\eqref{eq:popsH} and~\eqref{eq:popsB} can be decomposed accordingly.
To make this explicit, we first rewrite $\signedHet(k)$ and $\signedBias(k)$ in terms of the treatment effect function $\tau(x)$ and the main outcome function $\mu(x)$ as follows:

\begin{equation}\label{eq:sHvar}
\begin{aligned}
    \signedHet(k)
    = \sqrt{k (1-k)} \cdot \Bigl[&\bigl(
    \expectation{\mu(X) + \tfrac{1}{2} \tau(X) \mid X_1 \leq k,\, X \in R(\node),\, T=1}\\
    &\phantom{\bigl(}-
    \expectation{\mu(X) - \tfrac{1}{2} \tau(X) \mid X_1 \leq k,\, X \in R(\node),\, T=0}
    \bigr) \\
    -
    &\bigl(
    \expectation{\mu(X) + \tfrac{1}{2} \tau(X) \mid X_1 > k,\, X \in R(\node),\, T=1}\\
    &\phantom{\bigl(}-
    \expectation{\mu(X) - \tfrac{1}{2} \tau(X) \mid X_1 > k,\, X \in R(\node),\, T=0}
    \bigr)\Bigr] .
\end{aligned}
\end{equation}
and
\begin{equation}\label{eq:sBvar}
\begin{aligned}
    \signedBias(k)
    &= 
    \expectation{\mu(X) + \tfrac{1}{2} \tau(X) \mid X \in R(\node),\, T=1}\\
    &\phantom{=}-
    \expectation{\mu(X) - \tfrac{1}{2} \tau(X) \mid X \in R(\node),\, T=0} \\
    &-
    k \cdot
    \Bigl[
    \expectation{\mu(X) + \tfrac{1}{2} \tau(X) \mid X_1 \leq k,\, X \in R(\node),\, T=1}\\
    &\phantom{-k\cdot\Bigl[}-\expectation{\mu(X) - \tfrac{1}{2} \tau(X) \mid X_1 \leq k,\, X \in R(\node),\, T=0}
    \Bigr] \\
    &-
    (1-k) \cdot
    \Bigl[
    \expectation{\mu(X) + \tfrac{1}{2} \tau(X) \mid X_1 > k,\, X \in R(\node),\, T=1}\\
    &\phantom{-(1-k)\cdot\Bigl[}-
    \expectation{\mu(X) - \tfrac{1}{2} \tau(X) \mid X_1 > k,\, X \in R(\node),\, T=0}
    \Bigr] .
\end{aligned}
\end{equation}

Consequently, under the assumptions of Theorem~\ref{lem:decomposition}, we obtain the decomposition
\begin{equation}\label{eq:decomposeHetBias}
    \signedHet(k) = \signedHet_B(k) + \signedHet_H(k)
    \quad \text{and} \quad
    \signedBias(k) = \signedBias_B(k) + \signedBias_H(k),
\end{equation}
where $\signedHet_B(k)$ is defined analogously to~\eqref{eq:sHvar}, with $\mu$ and $\tau$ replaced by $\mu^B$ and $\tau^B$ as in Theorem~\ref{lem:decomposition}, and where $\signedHet_H(k)$, $\signedBias_B(k)$, and $\signedBias_H(k)$ are defined in complete analogy.
Note that, while there is some freedom in the choice of decomposition in Theorem~\ref{lem:decomposition}, any such decomposition results in the same values for $\signedHet_B(k)$, $\signedHet_H(k)$, $\signedBias_B(k)$, and $\signedBias_H(k)$.
An illustrative example of this decomposition for both the bias and heterogeneity splitting criteria is shown in Figure~\ref{Fig:criteria_decomposition}.

\begin{figure}[tb]
  \centering
  \includegraphics[height=0.8\textwidth]{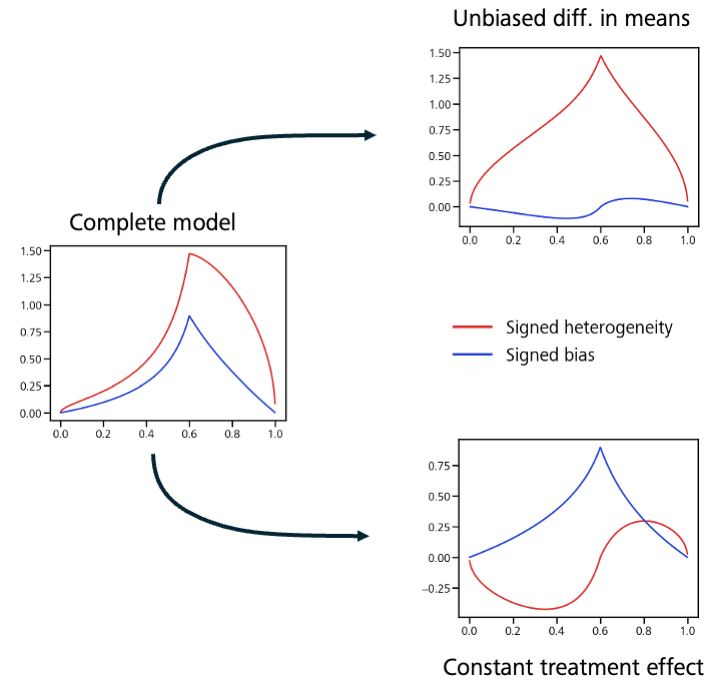}
   \caption{The signed heterogeneity and bias criteria, $\signedHet(k)$ and $\signedBias(k)$, together with their decomposition into $\signedHet_B(k)+\signedHet_H(k)$ and $\signedBias_B(k)+\signedBias_H(k)$ as in~\eqref{eq:decomposeHetBias}, for the model shown in Figure~\ref{Fig:model_decomposition}.}
  \label{Fig:criteria_decomposition}
\end{figure}

\begin{theorem}\label{theo:decompMax}
Consider a potential outcome model $(Y^{T = 0}, Y^{T = 1}, X, T)$ with a single uniformly distributed covariate \( X \sim U([a, b]), a<b \in \R \), potential outcomes as in~\eqref{Eq:outcomes_from_model}, and main outcome, treatment effect, and propensity following a change point model as in~\eqref{eq:change_point_model}, i.e.,
\[
\mu(x) = \alpha_{0} + \alpha_{1} \cdot \indicator{x \leq \gamma},\quad
\tau(x) = \beta_{0} + \beta_{1} \cdot \indicator{x \leq \gamma},\quad
p(x) = q_{0} + q_{1} \cdot \indicator{x \leq \gamma}.
\]
Fix a decomposition as in Theorem~\ref{lem:decomposition}, and define $\signedHet_B(k)$, $\signedHet_H(k)$, $\signedBias_B(k)$, and $\signedBias_H(k)$ as in~\eqref{eq:decomposeHetBias}.
Then, under Assumption~A3, the following statements hold:
\begin{enumerate}
    \item If the difference-in-means estimate is biased as in~\eqref{eq:unbiased}, i.e.,
    \[
        \expectation{\tau(X)} \neq \expectation{Y^{T = 1} \mid T = 1} - \expectation{Y^{T = 0} \mid T = 0},
    \]
    then $\signedBias_B(k)$ attains a global extremum at $k = \gamma$ and \( \signedHet_B(\gamma) = 0 \).
    \item If the treatment effect $\tau(x)$ is non-constant, then $\signedHet_H(k)$ attains a local extremum at $k = \gamma$ and \( \signedBias_H(\gamma) = 0 \).
\end{enumerate}
\end{theorem}
The proof is given in the appendix.

\begin{remark}
By the relationship between signed and unsigned splitting criteria, as well as the details of the proof of Theorem~\ref{theo:decompMax}, the extrema in Theorem~\ref{theo:decompMax} of the signed splitting criteria correspond to maxima of their unsigned counterparts while zeros are preserved. As the splitting criteria are non-negative, the zeros correspond to minima of the splitting criteria.
\end{remark}

The interpretation of Theorem~\ref{theo:decompMax} is as follows.
Theorem~\ref{lem:decomposition} shows that, locally at any fixed node, the potential outcome model marginalized with respect to a given splitting variable can be decomposed into two distinct components: one component that captures all heterogeneity in the treatment effect, and another component that captures all bias of the local difference-in-means estimator.
This decomposition is intrinsic to the marginalization mechanism of decision trees and holds at the population level for each candidate split.
As a consequence, a split at a given node may improve the local estimation of heterogeneous treatment effects for two fundamentally different reasons: either because it captures genuine heterogeneity in the treatment effect function, or because it reduces bias in the local average treatment effect induced by confounding.
Standard decision tree–based approaches do not distinguish between these two sources of improvement, as they rely only on the heterogeneity criterion when selecting splits.

Theorem~\ref{theo:decompMax} shows that this separation can be achieved at the population level by the proposed splitting criteria.
More precisely, the heterogeneity component of the signed heterogeneity criterion, $\signedHet_H(k)$, attains an extremum at the true jump location $\gamma$ whenever the treatment effect is non-constant, while the bias component of the signed bias criterion, $\signedBias_B(k)$ is maximized at the same location whenever bias is present.
The behavior predicted by Theorem~\ref{theo:decompMax} can also be observed empirically in the illustrative example shown in Figure~\ref{Fig:criteria_decomposition}.
Thus, the population versions of our splitting criteria $\signedHet(k)$ and $\signedBias(k)$ in~\eqref{eq:popsH} and~\eqref{eq:popsB} recover the split positions associated with treatment effect heterogeneity and bias, respectively, within the appropriate model components, with the former guaranteed up to a local optimum.
Overall, this result provides a theoretical explanation for the empirical improvements observed for the proposed splitting criteria in our simulation studies, presented in the next section.

\section{Simulations and Application}\label{Sec:simulations}
In this section, we apply IntCF on semi-synthetic and synthetic data sets, which allow for evaluation of prediction accuracy by having known ground-truth causal effects, as well as on a real data example. For synthetic data sets we are further able to compare importance measures to the relation of features to outcomes defined by the data model. Throughout, we pay particular attention to the trade-off between prediction accuracy, model simplicity, and interpretability, since this trade-off is the central motivation for our approach.

\subsection{Prediction Accuracy on Real-data Benchmarks}\label{Sec:causalpfn}

First, we compare prediction accuracy of our method with other established tree-based treatment effect estimation methods on four established data sets using the framework of \citet{balazadeh2025causalpfn}. As tree-based competitors we consider generalized random forests (GRF) \citep{athey_generalized_2019}, CausalForestDML as implemented in the EconML package \citep{econml}, and causal forest \citep{wager_estimation_2018}.
These models were employed to gain predictions for average treatment effect (ATE) and CATE on four data sets with ground-truth causal effects: IHDP \citep{ramey_infant_1992,hill_bayesian_2011}, ACIC 2016 \citep{dorie_automated_2019}, Lalonde CPS and PSID cohorts \citep{lalonde_evaluating_1986} with causal effects from RealCause \citep{neal_realcause_2021}. The predictions were evaluated against ground-truth ATE and CATE using relative ATE error and precision in estimation of heterogeneous treatment effects (PEHE) \citep[see][]{balazadeh2025causalpfn}.

Results are shown in Table~\ref{tab:causal_effect_results}. Overall, IntCF is mostly comparable with causal forest, with some improvement on data sets IHDP and ACIC 2016, while causal forest produces better ATE estimates on the Lalonde data sets. The results for the other two methods follow a different pattern: they perform better in estimating ATE for data sets IHDP and ACIC 2016, but their results for PEHE and ATE on the Lalonde data sets are significantly worse than those of IntCF. Overall, IntCF performs best in two settings, is second best in three settings, and third out of four in the remaining three settings, so that we conclude that IntCF is very competitive in terms of prediction accuracy compared to other established tree-based competitors.

Beyond the tree-based competitors considered here, \citep{balazadeh2025causalpfn} also benchmark several additional, non-tree-based methods on the same four data sets, using the identical evaluation setup we adopt here; we refer to \citet{balazadeh2025causalpfn} for the full comparison. Some of these methods, most notably CausalPFN, achieve higher prediction accuracy than any of the tree-based methods considered in Table~\ref{tab:causal_effect_results}, including IntCF, in most settings. However, CausalPFN and related methods are complete black boxes: unlike tree-based approaches, they offer no direct way to inspect which features drive a given prediction, or to what extent, and hence cannot answer the type of interpretability questions that motivate this paper. Tree-based methods therefore occupy a useful middle ground between prediction accuracy and interpretability. Within this class of methods, IntCF is particularly attractive: as shown above, it is competitive in prediction accuracy with the best tree-based alternatives, while, as we show in the remainder of this section, achieving substantially better interpretability than all of them.

\begin{table}[tb]
\centering
\begin{adjustbox}{max width=\textwidth}
\LARGE
\begin{tabular}{lcccc|cccc}
\toprule
\multirow{2}{*}{\textbf{Method}} &
\multicolumn{4}{c|}{\textbf{Mean PEHE $\pm$ Standard Error} $(\downarrow \text{better})$} &
\multicolumn{4}{c}{\textbf{Mean ATE Relative Error $\pm$ Standard Error} $(\downarrow \text{better})$} \\
\cmidrule(r){2-5}\cmidrule(l){6-9}
& IHDP & ACIC 2016 & \multicolumn{1}{c}{Lalonde {\footnotesize CPS}} & \multicolumn{1}{c|}{Lalonde {\footnotesize PSID}} & 
IHDP & ACIC 2016 & Lalonde {\footnotesize CPS} & Lalonde {\footnotesize PSID} \\
& & & ($\times10^{3}$) & ($\times10^{3}$) & & & & \\
\midrule

\textbf{IntCF} &
\firstbest{2.99$\pm$0.49} &
1.57$\pm$0.10 &
\secondbest{9.99$\pm$0.06} &
\firstbest{15.93$\pm$0.27} &
0.21$\pm$0.04 &
0.15$\pm$0.04 &
\secondbest{0.39$\pm$0.02} &
\secondbest{0.18$\pm$0.02} \\

\midrule

causal forest &
3.72$\pm$0.61 &
2.19$\pm$0.12 &
\firstbest{9.58$\pm$0.03} &
\secondbest{16.02$\pm$0.17} &
0.22$\pm$0.04 &
0.28$\pm$0.06 &
\firstbest{0.24$\pm$0.01} &
\firstbest{0.09$\pm$0.01} \\

\midrule

GRF &
\secondbest{3.67$\pm$0.61}& 
\firstbest{1.32$\pm$0.30}& 
12.33$\pm$0.06& 
22.91$\pm$0.17& %% CATE ENDS
\secondbest{0.18$\pm$0.03}& 
\secondbest{0.07$\pm$0.02}& 
0.82$\pm$0.02& 
0.85$\pm$0.02\\

Forest DML&
4.53$\pm$0.73& 
\secondbest{1.48$\pm$0.31}& 
12.95$\pm$0.04& 
22.99$\pm$0.15&  %% CATE ENDS
\firstbest{0.08$\pm$0.01}&
\firstbest{0.05$\pm$0.01}& 
1.03$\pm$0.01& 
1.05$\pm$0.01\\

\bottomrule
\end{tabular}
\end{adjustbox}
\caption{\textbf{CATE \& ATE results.} Columns correspond to benchmark suites: IHDP, ACIC~2016, Lalonde {\tiny CPS/PSID}. \emph{(left half)} mean PEHE, \emph{(right half)} mean ATE relative error. Lalonde PEHE is in thousands. The \firstbest{best} and \secondbest{second best} entries in each column are highlighted. This table is modified from \citet{balazadeh2025causalpfn}, with results for IntCF in the first line, results for causal forest \citep{wager_estimation_2018} in the second line and the results for GRF and Forest DML of \citet{balazadeh2025causalpfn} as comparison in the remaining two lines.} 
\label{tab:causal_effect_results}
\end{table}

\subsection{Interpretation Accuracy on Synthetic Data}\label{Sec:synthetic_sims}
While the data sets in Section~\ref{Sec:causalpfn} allow for evaluation of prediction accuracy by providing ground-truth causal effects, they do not lend themselves to evaluation of interpretability of models, as the mechanisms behind these effects are not known. For this reason, we generated synthetic data to compare importance measures to the structure of treatment effects as defined by the data model.

Four data models were used to generate the synthetic data sets for this section, each with a complexity parameter \( s \). This parameter scaled the dimension of the data, so that features for all models are generated uniformly from \( [0, 1]^d \) with \( d=5s \), and the parameter also affected the functions \( \mu, \tau, p \) defining the models. These models were defined as follows:
\begin{enumerate}
\item Interaction model: An adaption of a, so-called, locally-spiky-sparse (LSS) model, based on \citep{basu_2018, behr_provable_2022},
  with interactions for the treatment effect, but constant baseline and randomized treatment assignment:
  \begin{align*}
    p(X) &= 0.5, \\
    \mu(X) &= 0, \\
    \tau(X) &= \sum_{i=1}^{s} \indicator{X_{2i-1} \leq \gamma} \cdot \indicator{X_{2i} \leq \gamma},
  \end{align*}
  where \( \gamma = 0.7 \). The features responsible for treatment effect heterogeneity are therefore $X_1,\dotsc,X_{2s}$, appearing in interacting pairs $(X_{2i-1},X_{2i})$; since $p$ and $\mu$ do not depend on $X$ at all, this model contains no confounders. Note that the marginal effect (see Section~\ref{Sec:marginalization}) of any single feature in this LSS model again corresponds to a change point model, directly connecting this simulation setting to our theoretical analysis in Section~\ref{Sec:theory}.
  \item Change point model: An additive change point model with both confounding features and features responsible for treatment effect heterogeneity:
  \begin{align*}
    p(X) &= 0.2 + \frac{0.6}{s} \sum_{i=s+1}^{2s}\indicator{X_{i} \leq \gamma}, \\
    \mu(X) &= \sum_{i=s+1}^{2s} \indicator{X_{i} \leq \gamma}, \\
    \tau(X) &= \sum_{i=1}^{s} \indicator{X_i \leq \gamma},
  \end{align*}
  this time with \( \gamma = 0.5 \). The features responsible for treatment effect heterogeneity are therefore $X_1,\dotsc,X_s$, while $X_{s+1},\dotsc,X_{2s}$ act as pure confounders, affecting both the propensity and the main outcome but not the treatment effect itself.
  \item Linear model: A model with linear functions for propensity, main effect and treatment effect:
  \begin{align*}
    p(X) &= 0.25 + \frac{0.5}{s} \sum_{i=s+1}^{2s} \randomfeatures_i, \\
    \mu(X) &= \sum_{i=s+1}^{2s} \randomfeatures_i, \\
    \tau(X) &= \sum_{i=1}^{s} \randomfeatures_i.
  \end{align*}
  As in the change point model, the features responsible for treatment effect heterogeneity are $X_1,\dotsc,X_s$, while $X_{s+1},\dotsc,X_{2s}$ act as pure confounders.
  \item Mixed model: Combination of linear and LSS model:
  \begin{align*}
    p(X) &= 0.2 + \frac{0.6}{s} \sum_{i=2s+1}^{3s} \indicator{\randomfeatures_i \leq \gamma}, \\
    \mu(X) &= \sum_{i=2s+1}^{3s} \randomfeatures_i, \\
    \tau(X) &= \sum_{i=1}^{s} \indicator{\randomfeatures_{2i-1} \leq \gamma} \cdot \indicator{\randomfeatures_{2i} \leq \gamma} + \frac{1}{s} \sum_{i=s+1}^{2s} \randomfeatures_i,
  \end{align*}
  with \( \gamma = 0.7 \). The features responsible for treatment effect heterogeneity are therefore $X_1,\dotsc,X_{2s}$, with $X_{s+1},\dotsc,X_{2s}$ entering both the interaction and the linear part of $\tau$; $X_{2s+1},\dotsc,X_{3s}$ act as pure confounders.
\end{enumerate}

In all four models, any remaining features beyond those listed above, up to the total dimension $d=5s$, are pure noise, unrelated to $p$, $\mu$, or $\tau$. The ROC-AUC statistic which we introduced below treats a feature as a positive only if it appears in the formula for $\tau$; confounders and noise features are therefore both treated as negatives, even though confounders do affect $p$ and $\mu$.
Data was generated with normal distributed noise terms where the variance was equal to the variance of treatment effects.

Again, we compare our method IntCF to the three other tree-based causal machine learning methods, causal forest \citep{wager_estimation_2018}, GRF \citep{athey_generalized_2019} and CausalForestDML from EconML \citep{econml}. For the evaluation, both predictions and feature importance generated by these models were considered. Both GRF and CausalForestDML use a maximum depth and decay exponent for their feature importance calculation. We set maximum depth to infinity (or a very high value if not possible) and decay exponent to zero to ensure that the importance measures are comparable with those of the other estimation methods. For predictions, the precision in estimation of heterogeneous treatment effects (PEHE) was calculated and normalized by dividing it by the PEHE of a dummy predictor, which predicts a constant treatment effect by the difference-in-means from the training samples. For the feature importance, we considered a ROC-AUC statistic which measures how good the feature importance can distinguish between features which actually appear in the respective formula for treatment effect \( \tau \) and those which do not.

Each of the models with parameter \( s \) from 1 to 10 were used for simulations. For each model and parameter, $500 \cdot s$ training samples and 100\,000 separate evaluation samples were generated. The training samples were then used to train the machine learning models, which were subsequently evaluated on the evaluation samples. For each combination of data model, parameter \( s \), and machine learning method, these steps were repeated 20 times.

\begin{figure}[tb]
  \centering
  \includegraphics[width=0.9\textwidth]{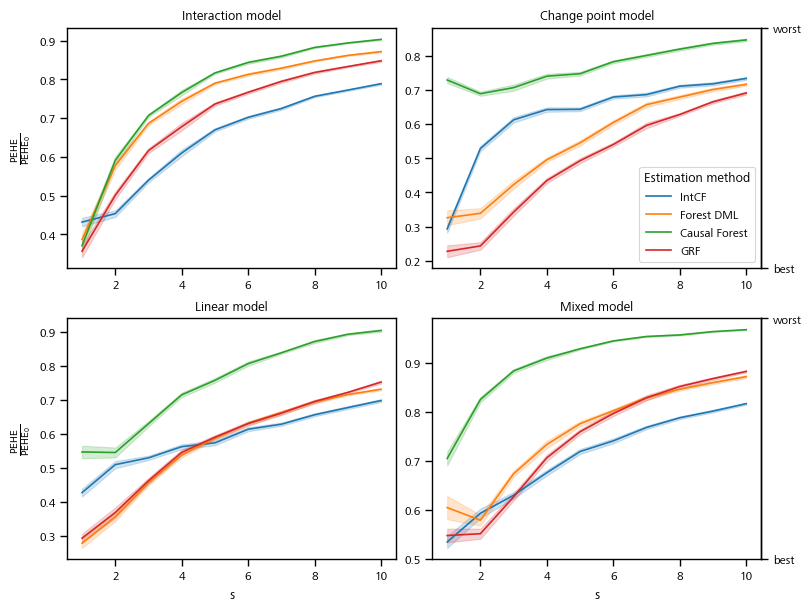}
  \caption{Comparison of precision in estimation of heterogeneous effects (PEHE) of IntCF (blue), Forest DML (yellow), causal forest (green) and GRF (red), normalized to PEHE of a constant difference-in-means estimator. For each method, a band of one standard error is around the respective mean. In this figure, lower values are better. Top left shows interaction model, top right change point model, bottom left linear model, and bottom right mixed model.}
  \label{Fig:pehe-synthetic_models}
\end{figure}
The results of these simulations for precision in estimation are shown in Figure~\ref{Fig:pehe-synthetic_models}.
In most considered data models IntCF achieves the best prediction results, especially for the Interaction and Mixed model. For the Change point model, Forest DML and GRF, both of which include a double machine learning/orthogonalization step, are the best performing methods. Between the two methods that do not use this additional step, IntCF has a better precision than causal forest in all cases except one (Interaction model with \( s=1 \)).

\begin{figure}[tb]
  \centering
  \includegraphics[width=0.9\textwidth]{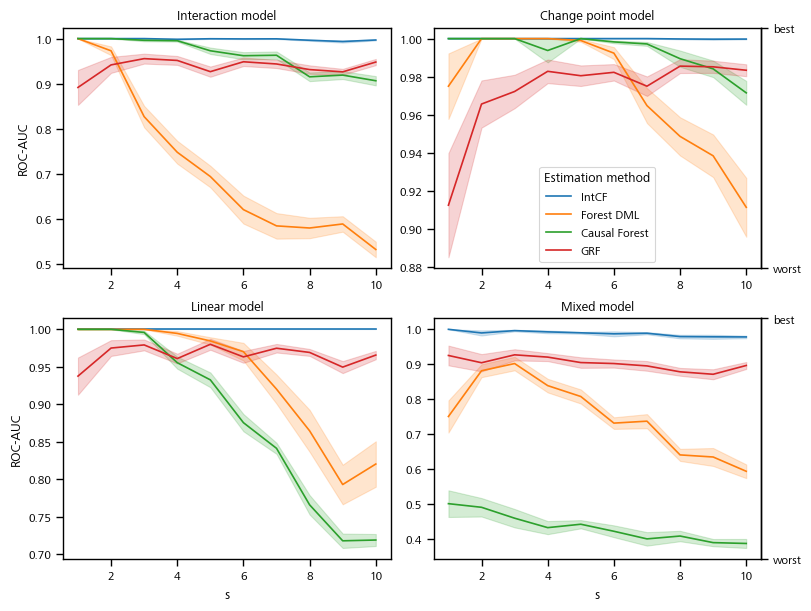}
  \caption{ROC-AUC for feature importance as signal feature classifier for IntCF (blue), Forest DML (yellow), causal forest (green) and GRF (red). For each method, a band of one standard error is around the respective mean. In this figure, higher values are better.  Top left shows interaction model, top right change point model, bottom left linear model, and bottom right mixed model.}
  \label{Fig:roc_auc-synthetic_models}
\end{figure}
Simulation results regarding feature importance are shown in Figure~\ref{Fig:roc_auc-synthetic_models}. In all considered combinations of data model and complexity parameter \( s \), the feature importance of IntCF results in a ROC-AUC of almost 1, while each other method achieves smaller values in most situations. This reflects a substantial improvement in interpretability, as IntCF is best at identifying the features actually relevant for treatment effect heterogeneity.
This improvement is particularly notable given the relative simplicity of IntCF. As discussed in Section~\ref{Sec:algorithm}, IntCF retains the same structure as the original RF: predictions are obtained directly as the average of individual decision trees, with no additional correction step required, and feature importance can be read off directly from the tree structure. Causal forest shares this simple structure, but does not distinguish heterogeneity from bias when selecting splits, and consequently performs worse in our feature importance comparison. GRF and CausalForestDML, in contrast, rely on a substantially more involved procedure: both incorporate an additional orthogonalization (double machine learning) step, which alters the prediction mechanism and means that predictions can no longer be read off directly from the individual trees, unlike for IntCF and causal forest. Despite this added complexity, GRF and CausalForestDML do not achieve better feature importance scores than IntCF in our simulations. IntCF thus achieves the strongest interpretability results of all four methods while remaining, by construction, the structurally simplest.

\subsection{Application on NHEFS Data set}
\begin{figure}[tb]
  \centering\includegraphics[width=0.6\textwidth]{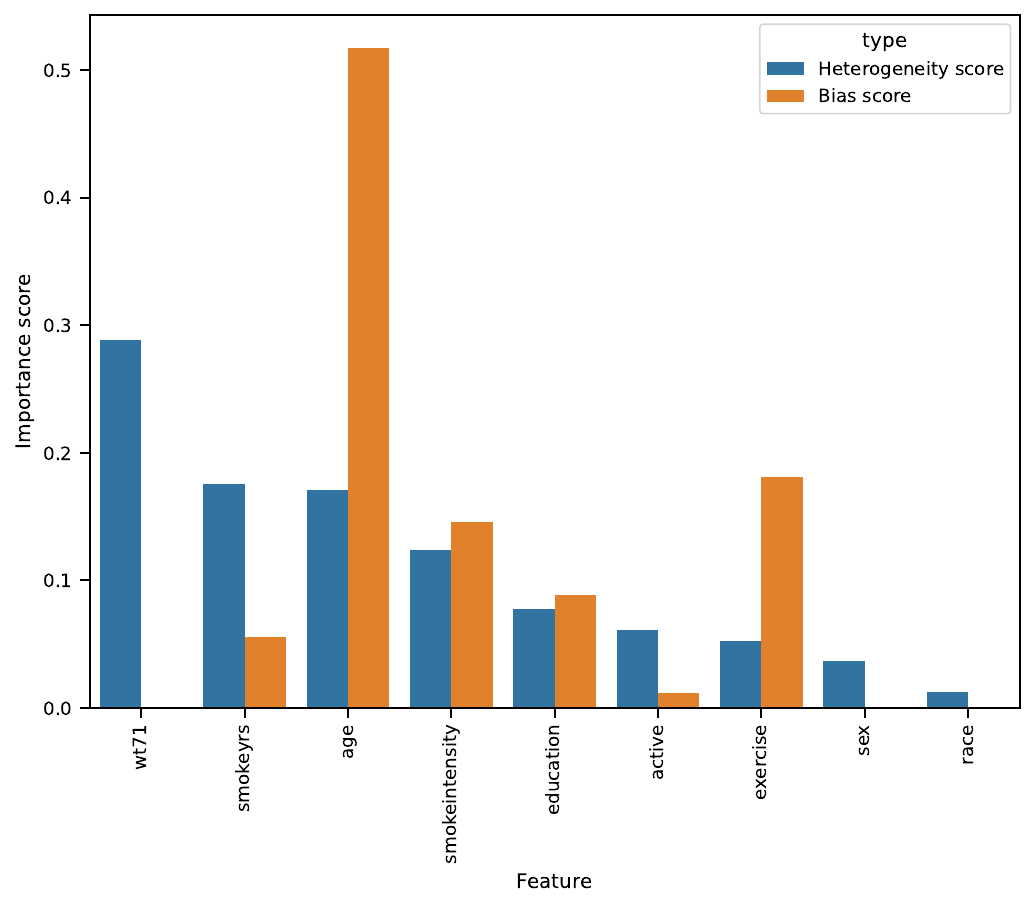}
  \caption{Importance scores for heterogeneity (blue) and bias (orange) produced by IntCF for NHEFS data set. Features are sorted by heterogeneity score. For the bias scores, negative values where truncated at 0.}\label{fig:nhefs}
\end{figure}
Finally, we demonstrate the application of our algorithm on a real-world data set: the National Health and Nutrition Examination Survey I Epidemiologic Followup Study (NHEFS)\footnote{More information about this study can be found at \url{https://wwwn.cdc.gov/nchs/nhanes/nhefs/}.} is a longitudinal study of a cohort first examined in 1971--75. This data set was also used previously to exemplify the handling of confounders and treatment effect heterogeneity \citep{hernan_causal_2020}. Like \citet{hernan_causal_2020}, we consider data from 1566 smokers to estimate the effect of quitting smoking on weight change up to the follow-up examination in 1982--84, using the same 9 baseline covariates as \citet{hernan_causal_2020}, including, e.g., the weight at the initial examination (\texttt{wt71}), the number of years of smoking (\texttt{smokeyrs}), and age (\texttt{age}). Notably, our method does not require these covariates to be labeled a priori as confounders or effect modifiers; instead, it attributes each split to bias correction or heterogeneity directly from the fitted tree.

For the average treatment effect, we obtain an estimated weight gain of \( 3.2\,\mathrm{kg} \) due to quitting smoking, slightly lower than the \( 3.4\,\mathrm{kg} \) reported by \citet{hernan_causal_2020}. Additionally, we evaluate feature importance scores for the contribution to heterogeneity and bias of the predictions, shown in Figure~\ref{fig:nhefs}.
As can be seen in Figure~\ref{fig:nhefs}, the feature with the highest heterogeneity importance is \texttt{wt71}, the weight at the initial examination, consistent with other methods that identify this feature as having the strongest influence on heterogeneity \citep[see, e.g.,][]{benard_2025}. In addition, our bias importance score identifies \texttt{age} as the most relevant confounder.

\section{Discussion}\label{Sec:discussion}

\subsection{Summary}
In this paper, we introduced the interpretable causal forest (IntCF), a modification of RF for estimating individual treatment effects. The only change relative to standard RF is the splitting criterion: by combining a heterogeneity criterion with a novel bias criterion, each split can be attributed to either genuine heterogeneity in the treatment effect or a correction for confounding bias. Together with a validation step that honestly re-evaluates these contributions at each node (Section~\ref{Sec:validation-of-splits}), this enables direct interpretation of the resulting tree structure, for example through standard feature importance measures. We supported this splitting criterion theoretically by analyzing its population limit under a change point model tailored to causal inference, showing that it correctly attributes a split to bias correction precisely when the naive difference-in-means estimator is biased, and to heterogeneity precisely when the treatment effect is genuinely non-constant. In simulations, IntCF achieved prediction accuracy competitive with established tree-based methods, while substantially outperforming all of them in recovering the features responsible for treatment effect heterogeneity, especially in the presence of confounders.

\subsection{Simplicity as a Feature, not a Limitation}
A recurring theme of this paper is that IntCF requires no machinery beyond a modified splitting criterion: predictions are still obtained directly as the average of individual decision trees, exactly as in the original RF. This stands in contrast to generalized random forest and CausalForestDML, both of which rely on an additional orthogonalization (double machine learning) step that changes the prediction mechanism and severs the direct correspondence between the fitted trees and the resulting predictions. Causal forest shares IntCF's simple structure, but does not distinguish heterogeneity from bias at the split level, and correspondingly falls short of IntCF's interpretability in our simulations. That IntCF exceeds the interpretability of substantially more complex competitors, while remaining competitive with them in prediction accuracy, is in our view the central practical contribution of this paper: interpretability need not come at the cost of added model complexity.

\subsection{Limitations and Future Work}
Several open questions remain. On the theoretical side, our analysis of the splitting criteria is currently limited to the population level; extending these results to finite samples, to more general data-generating processes beyond the change point model, and to the validation step from Section~\ref{Sec:validation-of-splits} are natural next steps. Likewise, our analysis of marginalization effects in Section~\ref{Sec:marginalization} considered only a single confounder at a time; understanding how these effects interact in the presence of multiple, simultaneously confounding features warrants further investigation.

Another open question is the consistency of the resulting IntCF estimator, which we did not establish in this paper. We expect that an adaptation of the consistency proof for causal forest from \citet{wager_estimation_2018} should be feasible. Such a proof typically requires the conditional mean function $\expectation{Y \mid \randomfeatures=\features}$ to be Lipschitz continuous, an assumption violated by the discontinuous change point model used in our theoretical analysis. However, as tree depth increases, the fraction of samples affected by any such discontinuity should vanish, suggesting that this is a technical rather than a fundamental obstacle.

\section*{Code Availability}
An implementation of IntCF, together with code to reproduce all analyses and figures in this paper, is publicly available at \url{https://github.com/behr-group/intcf}.

\section*{Acknowledgments}
The project was funded by the Deutsche Forschungsgemeinschaft (DFG, German Research Foundation), project number 509149993, TRR 374.

\bibliography{intcf}

\appendix
\section{Proofs}
\subsection{Proofs of Section~\ref{Sec:algorithm}}\label{Append:proofsAlgo}

\begin{proof}[Proof of Lemma~\ref{lem:mse_to_squared_change}]
  In the following computation, the property \( \indexTraining_{\parent} = \indexTraining_{\leftchild} \sqcup \indexTraining_{\rightchild} \) is used in the second equation and Assumption~\eqref{eq:averagePredCausal} is used in the third equation:
  \begin{align*}
      &\frac{1}{N} \sum_{i\in\indexTraining_{\parent}} (\hat{\tau}_{\parent} - \tau_{i})^{2} - \frac{1}{N} \left(\sum_{i\in\indexTraining_{\leftchild}} (\hat{\tau}_{\leftchild} - \tau_{i})^{2} + \sum_{i\in\indexTraining_{\rightchild}} (\hat{\tau}_{\rightchild} - \tau_{i})^{2}\right) \\
      &= \frac{1}{N} \left(\sum_{i\in\indexTraining_{\parent}} \hat{\tau}_{\parent}^{2} - \sum_{i\in\indexTraining_{\parent}} 2 \hat{\tau}_{\parent} \tau_{i} + \sum_{i\in\indexTraining_{\parent}} \tau_{i}^{2}\right) \\ &\phantom{=}- \frac{1}{N} \left(\left(\sum_{i\in\indexTraining_{\leftchild}} \hat{\tau}_{\leftchild}^{2} - \sum_{i\in\indexTraining_{\leftchild}} 2 \hat{\tau}_{\leftchild} \tau_{i} + \sum_{i\in\indexTraining_{\leftchild}} \tau_{i}^{2}\right) + \left(\sum_{i\in\indexTraining_{\rightchild}} \hat{\tau}_{\rightchild}^{2} - \sum_{i\in\indexTraining_{\rightchild}} 2 \hat{\tau}_{\rightchild} \tau_{i} + \sum_{i\in\indexTraining_{\rightchild}} \tau_{i}^{2}\right)\right) \\
      &= \frac{1}{N} \left((n_{\leftchild} + n_{\rightchild}) \hat{\tau}_{\parent}^{2} - 2 \hat{\tau}_{\parent} \left(\sum_{i\in\indexTraining_{\leftchild}} \tau_{i} + \sum_{i\in\indexTraining_{\rightchild}} \tau_{i}\right) + \sum_{i\in\indexTraining_{\parent}} \tau_{i}^{2}\right) \\ &\phantom{=}- \frac{1}{N} \left(n_{\leftchild} \hat{\tau}_{\leftchild}^{2} - 2 \hat{\tau}_{\leftchild} \sum_{i\in\indexTraining_{\leftchild}} \tau_{i} + n_{\rightchild} \hat{\tau}_{\rightchild}^{2} - 2 \hat{\tau}_{\rightchild} \sum_{i\in\indexTraining_{\rightchild}} \tau_{i} + \sum_{i\in\indexTraining_{\parent}} \tau_{i}^{2}\right) \\
      &= \frac{1}{N} \left((n_{\leftchild} + n_{\rightchild}) \hat{\tau}_{\parent}^{2} - 2 \hat{\tau}_{\parent} (n_{\leftchild} \tau_{\leftchild} + n_{\rightchild} \tau_{\rightchild}) + \sum_{i\in\indexTraining_{\parent}} \tau_{i}^{2}\right) \\ &\phantom{=}- \frac{1}{N} \left(n_{\leftchild} \hat{\tau}_{\leftchild}^{2} - 2 n_{\leftchild} \hat{\tau}_{\leftchild}^{2} + n_{\rightchild} \hat{\tau}_{\rightchild}^{2} - 2 n_{\rightchild} \hat{\tau}_{\rightchild}^{2} + \sum_{i\in\indexTraining_{\parent}} \tau_{i}^{2}\right) \\
      &= \frac{1}{N} \left((n_{\leftchild} + n_{\rightchild}) \hat{\tau}_{\parent}^{2} - 2 \hat{\tau}_{\parent} (n_{\leftchild} \tau_{\leftchild} + n_{\rightchild} \tau_{\rightchild}) + n_{\leftchild} \hat{\tau}_{\leftchild}^{2} + n_{\rightchild} \hat{\tau}_{\rightchild}^{2}\right) \\
      &= \frac{1}{N} (n_{\leftchild} (\hat{\tau}_{\parent} - \hat{\tau}_{\leftchild})^{2} + n_{\rightchild} (\hat{\tau}_{\parent} - \hat{\tau}_{\rightchild})^{2}) \qedhere
  \end{align*}
\end{proof}

\begin{proof}[Proof of Lemma~\ref{lem:squared_change_to_het_criterion}]
  By using assumption~\eqref{eq:meanPredCausal} direct computation yields:
  \begin{align*}
    &\frac{1}{N} (n_{\leftchild} (\hat{\tau}_{\parent} - \hat{\tau}_{\leftchild})^{2} + n_{\rightchild} (\hat{\tau}_{\parent} - \hat{\tau}_{\rightchild})^{2}) \\
    &= \frac{1}{N} \left(n_{\leftchild} \left(\frac{n_{\leftchild} \hat{\tau}_{\leftchild} + n_{\rightchild} \hat{\tau}_{\rightchild}}{n_{\leftchild} + n_{\rightchild}} - \hat{\tau}_{\leftchild}\right)^{2} + n_{\rightchild} \left(\frac{n_{\leftchild} \hat{\tau}_{\leftchild} + n_{\rightchild} \hat{\tau}_{\rightchild}}{n_{\leftchild} + n_{\rightchild}} - \hat{\tau}_{\rightchild}\right)^{2}\right) \\
    &= \frac{1}{N} \left(n_{\leftchild} \left(\frac{-n_{\rightchild} \hat{\tau}_{\leftchild} + n_{\rightchild} \hat{\tau}_{\rightchild}}{n_{\leftchild} + n_{\rightchild}}\right)^{2} + n_{\rightchild} \left(\frac{n_{\leftchild} \hat{\tau}_{\leftchild} - n_{\leftchild} \hat{\tau}_{\rightchild}}{n_{\leftchild} + n_{\rightchild}}\right)^{2}\right) \\
    &= \frac{1}{N} \left(\frac{n_{\leftchild} \cdot n_{\rightchild}^{2}}{(n_{\leftchild} + n_{\rightchild})^{2}} \left(-\hat{\tau}_{\leftchild} + \hat{\tau}_{\rightchild}\right)^{2} + \frac{n_{\leftchild}^{2} \cdot n_{\rightchild}}{(n_{\leftchild} + n_{\rightchild})^{2}} \left(\hat{\tau}_{\leftchild} - \hat{\tau}_{\rightchild}\right)^{2}\right) \\
    &= \frac{n_{\leftchild} \cdot n_{\rightchild}}{N (n_{\leftchild} + n_{\rightchild})} \left(\frac{n_{\rightchild}}{n_{\leftchild} + n_{\rightchild}} \left(\hat{\tau}_{\leftchild} - \hat{\tau}_{\rightchild}\right)^{2} + \frac{n_{\leftchild}}{n_{\leftchild} + n_{\rightchild}} \left(\hat{\tau}_{\leftchild} - \hat{\tau}_{\rightchild}\right)^{2}\right) \\
    &= \frac{n_{\leftchild} \cdot n_{\rightchild}}{N (n_{\leftchild} + n_{\rightchild})} \left(\hat{\tau}_{\leftchild} - \hat{\tau}_{\rightchild}\right)^{2} \qedhere
  \end{align*}
\end{proof}

\begin{proof}[Proof of Lemma~\ref{lem:derivation_bias_criterion}]
  By direct computation:
  \begin{align*}
    &\frac{1}{N} (n_{\leftchild} (\hat{\tau}_{\parent} - \hat{\tau}_{\leftchild})^{2} + n_{\rightchild} (\hat{\tau}_{\parent} - \hat{\tau}_{\rightchild})^{2}) - \frac{n_{\leftchild} \cdot n_{\rightchild}}{N (n_{\leftchild} + n_{\rightchild})} (\hat{\tau}_{\leftchild} - \hat{\tau}_{\rightchild})^{2} \\
    &= \frac{1}{N} \left(n_{\leftchild} (\hat{\tau}_{\parent} - \hat{\tau}_{\leftchild})^{2} + n_{\rightchild} (\hat{\tau}_{\parent} - \hat{\tau}_{\rightchild})^{2} - \frac{n_{\leftchild} \cdot n_{\rightchild}}{n_{\leftchild} + n_{\rightchild}} (\hat{\tau}_{\leftchild} - \hat{\tau}_{\rightchild})^{2}\right) \\
    &= \frac{1}{N} \bigg(n_{\leftchild} (\hat{\tau}_{\parent}^{2} - 2 \hat{\tau}_{\parent} \hat{\tau}_{\leftchild} + \hat{\tau}_{\leftchild}^{2}) + n_{\rightchild} (\hat{\tau}_{\parent}^{2} - 2 \hat{\tau}_{\parent} \hat{\tau}_{\rightchild} + \hat{\tau}_{\rightchild}^{2}) \\&\phantom{= \frac{1}{N}\bigg(} - \frac{n_{\leftchild} \cdot n_{\rightchild}}{n_{\leftchild} + n_{\rightchild}} (\hat{\tau}_{\leftchild}^{2} - 2\hat{\tau}_{\leftchild}\hat{\tau}_{\rightchild} + \hat{\tau}_{\rightchild}^{2})\bigg) \\
    &= \frac{1}{N} \bigg((n_{\leftchild} + n_{\rightchild}) \hat{\tau}_{\parent}^{2} - 2 \hat{\tau}_{\parent} (n_{\leftchild} \hat{\tau}_{\leftchild} + n_{\rightchild} \hat{\tau}_{\rightchild}) \\&\phantom{= \frac{1}{N}\bigg(} + \frac{n_{\leftchild}^{2}}{n_{\leftchild} + n_{\rightchild}} \hat{\tau}_{\leftchild}^{2} + 2 \frac{n_{\leftchild} \cdot n_{\rightchild}}{n_{\leftchild} + n_{\rightchild}} \hat{\tau}_{\leftchild} \hat{\tau}_{\rightchild} + \frac{n_{\rightchild}^{2}}{n_{\leftchild} + n_{\rightchild}} \hat{\tau}_{\rightchild}^{2}\bigg) \\
    &= \frac{n_{\leftchild} + n_{\rightchild}}{N} \left(\hat{\tau}_{\parent}^{2} - 2 \hat{\tau} \frac{n_{\leftchild} \hat{\tau}_{\leftchild} + n_{\rightchild} \hat{\tau}_{\rightchild}}{n_{\leftchild} + n_{\rightchild}} + \left(\frac{n_{\leftchild} \hat{\tau}_{\leftchild} + n_{\rightchild} \hat{\tau}_{\rightchild}}{n_{\leftchild} + n_{\rightchild}}\right)^{2}\right) \\
    &= \frac{n_{\leftchild} + n_{\rightchild}}{N} \left(\hat{\tau}_{\parent} - \frac{n_{\leftchild} \hat{\tau}_{\leftchild} + n_{\rightchild} \hat{\tau}_{\rightchild}}{n_{\leftchild}+n_{\rightchild}}\right)^{2} \qedhere
  \end{align*}
\end{proof}

\begin{proof}[Proof of Lemma~\ref{lem:sum_of_validation}]
  At each node \( \node \) we have \( \hat{\tau}_{\node}^{\backward} = \frac{1}{n} \sum_{i\in\indexValidation_{\node}} \hat{\tau}_{i}^{\validaton} \), so
  \begin{align*}
    &\sum_{i\in\indexValidation_{\node}} \left(\hat{\tau}_{\node}^{\forward} - \hat{\tau}_{i}^{\validaton}\right)^{2} - \sum_{i\in\indexValidation_{\node}} \left(\hat{\tau}_{\node}^{\backward} - \hat{\tau}_{i}^{\validaton}\right)^{2} \\
    &= \sum_{i\in\indexValidation_{\node}} \left((\hat{\tau}_{\node}^{\forward})^{2} - 2 \hat{\tau}_{\node}^{\forward} \hat{\tau}_{i}^{\validaton} + (\hat{\tau}_{i}^{\validaton})^{2}\right) - \sum_{i\in\indexValidation_{\node}} \left((\hat{\tau}_{\node}^{\backward})^{2} - 2 \hat{\tau}_{\node}^{\backward} \hat{\tau}_{i}^{\validaton} + (\hat{\tau}_{i}^{\validaton})^{2}\right) \\
    &= n \left((\hat{\tau}_{\node}^{\forward})^{2} - (\hat{\tau}_{\node}^{\backward})^{2}\right) + 2 \left(\hat{\tau}_{\node}^{\backward} - \hat{\tau}_{\node}^{\forward}\right) \sum_{i\in\indexValidation_{\node}} \hat{\tau}_{i}^{\validaton} \\
    &= n \left((\hat{\tau}_{\node}^{\forward})^{2} - (\hat{\tau}_{\node}^{\backward})^{2}\right) + 2 \left(\hat{\tau}_{\node}^{\backward} - \hat{\tau}_{\node}^{\forward}\right) \cdot n \cdot \hat{\tau}_{\node}^{\backward} \\
    &= n \left((\hat{\tau}_{\node}^{\forward})^{2} - (\hat{\tau}_{\node}^{\backward})^{2} + 2 (\hat{\tau}_{\node}^{\backward})^{2} - 2\hat{\tau}_{\node}^{\forward}\hat{\tau}_{\node}^{\backward}\right) \\
    &= n \left((\hat{\tau}_{\node}^{\forward})^{2} - 2 \hat{\tau}_{\node}^{\forward}\hat{\tau}_{\node}^{\backward} + (\hat{\tau}_{\node}^{\backward})^{2}\right) \\
    &= n \left(\hat{\tau}_{\node}^{\forward} - \hat{\tau}_{\node}^{\backward}\right)^{2},
  \end{align*}
  and, using \( \hat{\tau}_{\parent}^{\backward} = \frac{n_{\leftchild} \hat{\tau}_{\leftchild}^{\backward} + n_{\rightchild} \hat{\tau}_{\rightchild}^{\backward}}{n_{\leftchild} + n_{\rightchild}} \), we get
  \begin{align*}
    & \sum_{i\in\indexValidation_{\parent}} \left(\hat{\tau}_{\parent}^{\backward} - \hat{\tau}_{i}^{\validaton}\right)^{2} - \sum_{i\in\indexValidation_{\leftchild}} \left(\hat{\tau}_{\leftchild}^{\backward} - \hat{\tau}_{i}^{\validaton}\right)^{2} - \sum_{i\in\indexValidation_{\rightchild}} \left(\hat{\tau}_{\rightchild}^{\backward} - \hat{\tau}_{i}^{\validaton}\right)^{2} \\
    &= \sum_{i\in\indexValidation_{\rightchild}} \left((\hat{\tau}_{\parent}^{\backward})^{2} - 2 \hat{\tau}_{\parent}^{\backward} \hat{\tau}_{i}^{\validaton} + (\hat{\tau}_{i}^{\validaton})^{2}\right) \\
    &\phantom{=} - \sum_{i\in\indexValidation_{\leftchild}} \left((\hat{\tau}_{\leftchild}^{\backward})^{2} - 2 \hat{\tau}_{\leftchild}^{\backward} \hat{\tau}_{i}^{\validaton} + (\hat{\tau}_{i}^{\validaton})^{2}\right) - \sum_{i\in\indexValidation_{\rightchild}} \left((\hat{\tau}_{\rightchild}^{\backward})^{2} - 2 \hat{\tau}_{\rightchild}^{\backward} \hat{\tau}_{i}^{\validaton} + (\hat{\tau}_{i}^{\validaton})^{2}\right) \\
    &= (n_{\leftchild}+n_{\rightchild}) (\hat{\tau}_{\parent}^{\backward})^{2} - 2 (n_{\leftchild}+n_{\rightchild}) (\hat{\tau}_{\parent}^{\backward})^{2} \\
    &\phantom{=} - n_{\leftchild} (\hat{\tau}_{\leftchild}^{\backward})^{2} + 2 n_{\leftchild} (\hat{\tau}_{\leftchild}^{\backward})^{2} - n_{\rightchild} (\hat{\tau}_{\rightchild}^{\backward})^{2} + 2 n_{\rightchild} (\hat{\tau}_{\rightchild}^{\backward})^{2} \\
    &= - (n_{\leftchild}+n_{\rightchild}) (\hat{\tau}_{\parent}^{\backward})^{2} + n_{\leftchild} (\hat{\tau}_{\leftchild}^{\backward})^{2} + n_{\rightchild} (\hat{\tau}_{\rightchild}^{\backward})^{2} \\
    &= n_{\leftchild} (\hat{\tau}_{\leftchild}^{\backward})^{2} + n_{\rightchild} (\hat{\tau}_{\rightchild}^{\backward})^{2} - (n_{\leftchild}+n_{\rightchild}) \left(\frac{n_{\leftchild} \hat{\tau}_{\leftchild}^{\backward} + n_{\rightchild} \hat{\tau}_{\rightchild}^{\backward}}{n_{\leftchild} + n_{\rightchild}}\right)^{2} \\
    &= \frac{(n_{\leftchild} + n_{\rightchild}) n_{\leftchild} (\hat{\tau}_{\leftchild}^{\backward})^{2} + (n_{\leftchild} + n_{\rightchild}) n_{\rightchild} (\hat{\tau}_{\rightchild}^{\backward})^{2} - n_{\leftchild}^{2} (\hat{\tau}_{\leftchild}^{\backward})^{2} - 2 n_{\leftchild} n_{\rightchild} \hat{\tau}_{\leftchild}^{\backward} \hat{\tau}_{\rightchild}^{\backward} - n_{\rightchild}^{2} (\hat{\tau}_{\rightchild}^{\backward})^{2}}{n_{\leftchild} + n_{\rightchild}} \\
    &= \frac{n_{\leftchild} \cdot n_{\rightchild}}{n_{\leftchild} + n_{\rightchild}} \left(\hat{\tau}_{\leftchild}^{\backward} - \hat{\tau}_{\rightchild}^{\backward}\right)^{2}.
  \end{align*}
  Combining these results, we have
  \begin{multline*}
    \sum_{i\in\indexValidation_{\parent}} \left(\hat{\tau}_{\parent}^{\forward} - \hat{\tau}_{i}^{\validaton}\right)^{2} - \sum_{i\in\indexValidation_{\leftchild}} \left(\hat{\tau}_{\leftchild}^{\forward} - \hat{\tau}_{i}^{\validaton}\right)^{2} - \sum_{i\in\indexValidation_{\rightchild}} \left(\hat{\tau}_{\rightchild}^{\forward} - \hat{\tau}_{i}^{\validaton}\right)^{2} =\\
    \frac{n_{\leftchild} \cdot n_{\rightchild}}{n_{\leftchild} + n_{\rightchild}} \left(\hat{\tau}_{\leftchild}^{\backward} - \hat{\tau}_{\rightchild}^{\backward}\right)^{2} + \frac{n_{\leftchild} + n_{\rightchild}}{N} \left(\hat{\tau}_{\parent}^{\forward} - \hat{\tau}_{\parent}^{\backward}\right)^{2} - \frac{n_{\leftchild}}{N} \left(\hat{\tau}_{\leftchild}^{\forward} - \hat{\tau}_{\leftchild}^{\backward}\right)^{2} - \frac{n_{\rightchild}}{N} \left(\hat{\tau}_{\rightchild}^{\forward} - \hat{\tau}_{\rightchild}^{\backward}\right)^{2}.
  \end{multline*}
\end{proof}

\subsection{Proof of Theorem~\ref{lem:decomposition}}\label{Append:ProofTheoDecomp}
\begin{remark}\label{Rem:unit_interval}
  In the following statements and proofs, we will assume w.l.o.g.\ that the covariate \( X \) is uniformly distributed on \( [0, 1] \), also implying \( R(\node) = [0, 1] \). If its domain would be a different interval \( [a, b] \), replacing it by \( \frac{X-a}{b-a} \)---as well as performing the same transformation for \( \gamma \) and \( k \)---would create an equivalent model with the same values for treatment effect estimates and (signed) splitting criteria.
\end{remark}

\begin{lemma}\label{lem:bias_of_prediction}
  For the univariate change-point model~\eqref{eq:change_point_model} with potential outcomes as in~\eqref{Eq:outcomes_from_model}, assuming Assumption~A3 holds, the bias of the difference-in-means estimate as defined in~\eqref{eq:unbiased} is given by 
\begin{multline}\label{Eq:prediction-diff}
  \left(\expectation{Y^{T = 1} \mid T = 1} - \expectation{Y^{T = 0} \mid T = 0} \right) - \expectation{\tau(X)}  \\ 
  = \gamma (1 - \gamma) q_{1} \frac{\alpha_{1} + (1/2 - q_{0} - \gamma q_{1}) \beta_{1}}{(q_{0} + \gamma q_{1}) (1 - q_{0} - \gamma q_{1})}.
\end{multline}
\end{lemma}
\begin{proof}
  Because of Assumption~A3, the conditional expectations are well defined and we have
  \begin{align*}
    \expectation{Y^{T=0} \mid T=0} &= \alpha_{0} - \beta_{0}/2 + \frac{(\alpha_{1} - \beta_{1}/2) \gamma (1 - q_{0} - q_{1})}{1 - q_{0} - \gamma q_{1}}, \\
    \expectation{Y^{T=1} \mid T=1} &= \alpha_{0} + \beta_{0}/2 + \frac{(\alpha_{1} + \beta_{1}/2) \gamma (q_{0} + q_{1})}{q_{0} + \gamma q_{1}}
  \end{align*}
  and hence, 
  \begin{align*}
   &\expectation{Y^{T = 1} \mid T = 1} - \expectation{Y^{T = 0} \mid T = 0}\\
   = &\beta_{0} + \frac{(\alpha_{1} + \beta_{1}/2) \gamma (q_{0} + q_{1})}{q_{0} + \gamma q_{1}} - \frac{(\alpha_{1} - \beta_{1}/2) \gamma (1 - q_{0} - q_{1})}{1 - q_{0} - \gamma q_{1}} \\
    =& \beta_{0} + \gamma (\alpha_{1} - \beta_{1}/2) \frac{(1 - \gamma) q_{1}}{(q_{0} + \gamma q_{1}) (1 - q_{0} - \gamma q_{1})} + \gamma \beta_{1} \left(1 + \frac{(1 - \gamma) q_{1}}{q_{0} + \gamma q_{1}}\right).
\end{align*}
Moreover, we have \( \expectation{\tau(X)} = \beta_{0} + \gamma \beta_{1} \) and hence,
\begin{align*}
    &\left(\expectation{Y^{T = 1} \mid T = 1} - \expectation{Y^{T = 0} \mid T = 0} \right) - \expectation{\tau(X)}  \\
    =& (\alpha_{1} - \beta_{1}/2) \frac{\gamma (1 - \gamma) q_{1}}{(q_{0} + \gamma q_{1}) (1 - q_{0} - \gamma q_{1})} + \beta_{1} \frac{\gamma (1 - \gamma) q_{1}}{q_{0} + \gamma q_{1}} \\
    =& \gamma (1 - \gamma) q_{1} \frac{\alpha_{1} + (1/2 - q_{0} - \gamma q_{1}) \beta_{1}}{(q_{0} + \gamma q_{1}) (1 - q_{0} - \gamma q_{1})}. \qedhere
  \end{align*}
\end{proof}

\begin{proof}[Proof for Theorem~\ref{lem:decomposition}]
  Define
  \begin{align*}
    \alpha_0^{\prime} &:= \alpha_0, & \alpha_1^{\prime} &:= \alpha_1 - \alpha_1^{\prime\prime}, & \beta_0^{\prime} &:= \beta_0, \\
    \alpha_0^{\prime\prime}&:=0, & \alpha_{1}^{\prime\prime} &:= -(1/2 - q_{0} - \gamma q_{1}) \beta_{1}, & \beta_0^{\prime\prime} &:= 0, & \beta_{1}^{\prime\prime} &:= \beta_{1}.
  \end{align*}
  With these definitions, \( \mu(x) = \mu^B(x) + \mu^H(x) \) and \( \tau(x) = \tau^B(x) + \tau^H(x) \) follows directly.

  For the second model, we consider the difference between the difference-in-means prediction \( \expectation{\tau^H(X)} = \expectation{\mu^H(X) + \frac{1}{2}\tau^H(X) \mid T = 1 } - \expectation{\mu^H(X) - \frac{1}{2}\tau^H(X) \mid T = 0 } \) and the expectation of the treatment effect \( \expectation{\tau^H(X)} \). By Lemma~\ref{lem:bias_of_prediction} this is given by
  \[ \gamma (1 - \gamma) q_{1} \frac{\alpha_{1}^{\prime\prime} + (1/2 - q_{0} - \gamma q_{1}) \beta_{1}^{\prime\prime}}{(q_{0} + \gamma q_{1}) (1 - q_{0} - \gamma q_{1})}=0, \]
  where for the evaluation we used \( \alpha_{1}^{\prime\prime} = -(1/2 - q_{0} - \gamma q_{1}) \beta_{1} \) and \( \beta_{1}^{\prime\prime} = \beta_{1} \). So the difference-in-means estimator is in expectation equal to the treatment effect, i.e., is unbiased.
\end{proof}

\subsection{Proof of Theorem~\ref{theo:decompMax}}\label{Append:ProofTheoDecompMax}
\begin{remark}
  We will continue assuming \( X \sim U([0, 1]) \), see Remark~\ref{Rem:unit_interval}.

  Further we will assume that \( \gamma \in (0, 1) \) in~\eqref{eq:change_point_model}. For values outside \( (0, 1) \), there is a model with \( \beta_1 = q_1 = \alpha_1 = 0 \) and \( \gamma \in (0, 1) \) which has almost surely the same propensities and expected outcomes, which therefore also result in the same expected predictions and (signed) splitting criteria.
\end{remark}

\begin{lemma}\label{lem:population_predictions}
For the univariate change-point model~\eqref{eq:change_point_model}, assuming Assumption~A3 holds, with potential outcomes as in~\eqref{Eq:outcomes_from_model} and a split point \( k \in (0, 1) \)
define
\begin{equation}\label{Eq:tauInf}
\begin{aligned}
        &\hat{\tau}^{\infty}_{\leftchild} := \expectation{Y^{T=1} \mid X \leq k, T=1}  - \expectation{Y^{T=0} \mid X \leq k, T=0},\\
        &    \hat{\tau}^{\infty}_{\rightchild} := \expectation{Y^{T=1} \mid X > k, T=1}  - \expectation{Y^{T=0} \mid X > k, T=0} 
\end{aligned}
\end{equation}
Then we have
  \[ \hat{\tau}^{\infty}_{\leftchild} = \begin{cases}
    \beta_{0} + \beta_{1} & k \leq \gamma,\\
    \beta_{0} + (\alpha_{1} - \beta_{1}/2) \frac{\gamma (k - \gamma) q_{1}}{(k q_{0} + \gamma q_{1}) (k - k q_{0} - \gamma q_{1})} + \beta_{1} \left(1 - \frac{(k - \gamma) q_{0}}{k q_{0} + \gamma q_{1}}\right) & k > \gamma
  \end{cases} \]
and
  \[ \hat{\tau}^{\infty}_{\rightchild} = \begin{cases}
    \begin{aligned}
        \beta_{0} &+ (\alpha_{1} - \beta_{1}/2) \frac{(\gamma - k) (1 - \gamma) q_{1}}{((1 - k) q_{0} + (\gamma - k) q_{1}) ((1 - k) (1 - q_{0}) - (\gamma - k) q_{1})} \\
        &+ \beta_{1} \left(1 - \frac{(1 - \gamma) q_{0}}{(1 - k) q_{0} + (\gamma - k) q_{1}}\right)
    \end{aligned} & k \leq \gamma,\\
    \beta_{0} & k > \gamma.
  \end{cases} \]
\end{lemma}
\begin{proof}
For a splitting point \( k \), we have
\begin{align*}
  \expectation{Y^{T=0} \mid X \leq k, T=0} &= \begin{cases}
    \alpha_{0} - \beta_{0}/2 + \alpha_{1} - \beta_{1}/2 & k \leq \gamma,\\
    \alpha_{0} - \beta_{0}/2 + \frac{(\alpha_{1} - \beta_{1}/2) \gamma (1 - q_{0} - q_{1})}{k - k q_{0} - \gamma q_{1}} & k > \gamma,
  \end{cases} \\
  \expectation{Y^{T=1} \mid X \leq k, T=1} &= \begin{cases}
    \alpha_{0} + \beta_{0}/2 + \alpha_{1} + \beta_{1}/2 & k \leq \gamma,\\
    \alpha_{0} + \beta_{0}/2 + \frac{(\alpha_{1} + \beta_{1}/2) \gamma (q_{0} + q_{1})}{k q_{0} + \gamma q_{1}} & k > \gamma,
  \end{cases}
\end{align*}
and hence,
\[
  \hat{\tau}^{\infty}_{\leftchild} = \begin{cases}
    \beta_{0} + \beta_{1} & k \leq \gamma,\\
    \beta_{0} + (\alpha_{1} - \beta_{1}/2) \frac{\gamma (k - \gamma) q_{1}}{(k q_{0} + \gamma q_{1}) (k - k q_{0} - \gamma q_{1})} + \beta_{1} \left(1 - \frac{(k - \gamma) q_{0}}{k q_{0} + \gamma q_{1}}\right) & k > \gamma.
  \end{cases}
\]
Moreover, we have
\begin{align*}
  \expectation{Y^{T=0} \mid X > k, T=0} &= \begin{cases}
    \alpha_{0} - \beta_{0}/2 + \frac{(\alpha_{1} - \beta_{1}/2) (\gamma - k) (1 - q_{0} - q_{1})}{(1-k)(1-q_{0}) - (\gamma-k)q_{1}} & k \leq \gamma,\\
    \alpha_{0} - \beta_{0}/2 & k > \gamma,
  \end{cases}\\
  \expectation{Y^{T=1} \mid X > k, T=1} &= \begin{cases}
    \alpha_{0} + \beta_{0}/2 + \frac{(\alpha_{1} + \beta_{1}/2) (\gamma - k) (q_{0} + q_{1})}{(1-k) q_{0} + (\gamma-k) q_{1}} & k \leq \gamma,\\
    \alpha_{0} + \beta_{0}/2 & k > \gamma,
  \end{cases}
\end{align*}
and hence,
\[
  \hat{\tau}^{\infty}_{\rightchild} = \begin{cases}
    \begin{aligned}
        \beta_{0} &+ (\alpha_{1} - \beta_{1}/2) \frac{(\gamma - k) (1 - \gamma) q_{1}}{((1 - k) q_{0} + (\gamma - k) q_{1}) ((1 - k) (1 - q_{0}) - (\gamma - k) q_{1})} \\
        &+ \beta_{1} \left(1 - \frac{(1 - \gamma) q_{0}}{(1 - k) q_{0} + (\gamma - k) q_{1}}\right)
    \end{aligned} & k \leq \gamma,\\
    \beta_{0} & k > \gamma.
  \end{cases}
\]
Note that all terms in the denominators are either treatment propensities for some subset of samples, scaled versions of these, or a product of multiple such propensities. By Assumption~A3, these are positive, so all terms are defined.
\end{proof}

\begin{proof}[Proof of Theorem~\ref{theo:decompMax}]
Use the notation from~\eqref{Eq:tauInf} and
\begin{multline*}
     \hat{\tau}^{\infty}_\parent := \expectation{Y^{T=1} \mid  T=1}  - \expectation{Y^{T=0} \mid  T=0} \\= \beta_{0} + \gamma (\alpha_{1} - \beta_{1}/2) \frac{(1 - \gamma) q_{1}}{(q_{0} + \gamma q_{1}) (1 - q_{0} - \gamma q_{1})} + \gamma \beta_{1} \left(1 + \frac{(1 - \gamma) q_{1}}{q_{0} + \gamma q_{1}}\right),
\end{multline*}
as calculated in Lemma~\ref{lem:population_predictions}. Note that we can write
\begin{align*}
  \signedHet(k) &= \sqrt{k (1-k)} \cdot (\hat{\tau}^{\infty}_{\leftchild} - \hat{\tau}^{\infty}_{\rightchild}), \\
  \signedBias(k) &= \hat{\tau}^{\infty}_{\parent} - k \hat{\tau}^{\infty}_{\leftchild} - (1-k) \hat{\tau}^{\infty}_{\rightchild}.
\end{align*}

Further note, that \( \alpha_0 \) does not appear in any of \( \hat{\tau}^{\infty}_{\parent}, \hat{\tau}^{\infty}_{\leftchild}, \hat{\tau}^{\infty}_{\rightchild} \), so it does not have any effect on \( \signedHet(k) \) or \( \signedBias(k) \). Additionally, \( \beta_0 \) appears in all three treatment effect estimates, but always as simple summand. So changing it does not have any effect on the differences \( \hat{\tau}^{\infty}_{\leftchild} - \hat{\tau}^{\infty}_{\rightchild} \) and \( \hat{\tau}^{\infty}_{\parent} - k \hat{\tau}^{\infty}_{\leftchild} - (1-k) \hat{\tau}^{\infty}_{\rightchild} \), which define \( \signedHet(k) \) and \( \signedBias(k) \). To simplify notation, we will therefore, without loss of generality, set all of \( \alpha_0, \alpha_0^{\prime}, \alpha_0^{\prime\prime}, \beta_0, \beta_0^{\prime}, \beta_0^{\prime\prime} \) zero for the following analysis of the signed splitting criteria.

We proof the two statements (for bias and for heterogeneity) of the theorem separately.

\paragraph{Bias}
After setting \( \alpha_0^{\prime} = \beta_0^{\prime} = 0 \), as explained above, the bias model is given by
\begin{align*}
    \mu^B(x) &= \alpha_{1}^{\prime} \cdot \indicator{x \leq \gamma}, \\
    \tau^B(x) &= 0, \\
    p(x) &= q_{0} + q_{1} \cdot \indicator{x \leq \gamma}.
\end{align*}

In this model, the estimated treatment effect tends to 
\[ \hat{\tau}^{\infty}_{\parent} = \gamma \alpha_{1}^{\prime} \frac{(1 - \gamma) q_{1}}{(q_{0} + \gamma q_{1}) (1 - q_{0} - \gamma q_{1})}, \] while the true treatment effect is constant zero: \( \tau^B = 0 \).\\
In the child nodes for a split at \( k \), the estimates are
\begin{align*}
  \hat{\tau}^{\infty}_{\leftchild} &= \begin{cases}
    0 & k \leq \gamma,\\
    \alpha_{1}^{\prime} \frac{\gamma (k - \gamma) q_{1}}{(k q_{0} + \gamma q_{1}) (k - k q_{0} - \gamma q_{1})} & k > \gamma,
  \end{cases} \\
  \hat{\tau}^{\infty}_{\rightchild} &= \begin{cases}
    \alpha_{1}^{\prime} \frac{(\gamma - k) (1 - \gamma) q_{1}}{((1 - k) q_{0} + (\gamma - k) q_{1}) ((1 - k) (1 - q_{0}) - (\gamma - k) q_{1})} & k \leq \gamma,\\
    0 & k > \gamma.
  \end{cases}
\end{align*}
Therefore, the population signed splitting criteria are given by
\begin{align*}
    \signedHet_B(k) &= \sqrt{k (1-k)} (\hat{\tau}^{\infty}_{\leftchild} - \hat{\tau}^{\infty}_{\rightchild}), \\
    &= \begin{cases}
        \sqrt{k (1-k)} \left(0 - \alpha_{1}^{\prime} \frac{(\gamma - k) (1 - \gamma) q_{1}}{((1 - k) q_{0} + (\gamma - k) q_{1}) ((1 - k) (1 - q_{0}) - (\gamma - k) q_{1})}\right) & k \leq \gamma, \\
        \sqrt{k (1-k)} \left(\alpha_{1}^{\prime} \frac{\gamma (k - \gamma) q_{1}}{(k q_{0} + \gamma q_{1}) (k - k q_{0} - \gamma q_{1})} - 0\right) & k > \gamma,
    \end{cases} \\
    &= \begin{cases}
        -\sqrt{k (1-k)} \alpha_{1}^{\prime} \frac{(\gamma - k) (1 - \gamma) q_{1}}{((1 - k) q_{0} + (\gamma - k) q_{1}) ((1 - k) (1 - q_{0}) - (\gamma - k) q_{1})} & k \leq \gamma, \\
        \sqrt{k (1-k)} \alpha_{1}^{\prime} \frac{(k - \gamma) \gamma q_{1}}{(k q_{0} + \gamma q_{1}) (k - k q_{0} - \gamma q_{1})} & k > \gamma,
    \end{cases} \\[3mm]
    \signedBias_B(k) &= \hat{\tau}^{\infty}_{\parent} - k \hat{\tau}^{\infty}_{\leftchild} - (1-k) \hat{\tau}^{\infty}_{\rightchild} \\
    &= \begin{cases}
        \scriptstyle \gamma \alpha_{1}^{\prime} \frac{(1 - \gamma) q_{1}}{(q_{0} + \gamma q_{1}) (1 - q_{0} - \gamma q_{1})} - k \cdot 0 - (1-k) \alpha_{1}^{\prime} \frac{(\gamma - k) (1 - \gamma) q_{1}}{((1 - k) q_{0} + (\gamma - k) q_{1}) ((1 - k) (1 - q_{0}) - (\gamma - k) q_{1})} & k \leq \gamma,\\
        \gamma \alpha_{1}^{\prime} \frac{(1 - \gamma) q_{1}}{(q_{0} + \gamma q_{1}) (1 - q_{0} - \gamma q_{1})} - k \alpha_{1}^{\prime} \frac{\gamma (k - \gamma) q_{1}}{(k q_{0} + \gamma q_{1}) (k - k q_{0} - \gamma q_{1})} - (1-k) \cdot 0 & k > \gamma,
    \end{cases} \\
    &= \begin{cases}
        \scriptstyle \alpha_{1}^{\prime} \left(\frac{\gamma (1 - \gamma) q_{1}}{(q_{0} + \gamma q_{1}) (1 - q_{0} - \gamma q_{1})} - (1-k) \frac{(\gamma - k) (1 - \gamma) q_{1}}{((1 - k) q_{0} + (\gamma - k) q_{1}) ((1 - k) (1 - q_{0}) - (\gamma - k) q_{1})}\right) & k \leq \gamma,\\
        \alpha_{1}^{\prime} \left(\frac{\gamma (1 - \gamma) q_{1}}{(q_{0} + \gamma q_{1}) (1 - q_{0} - \gamma q_{1})} - k \frac{\gamma (k - \gamma) q_{1}}{(k q_{0} + \gamma q_{1}) (k - k q_{0} - \gamma q_{1})}\right) & k > \gamma,
    \end{cases}
\end{align*}
Note that $\signedHet_B(k)$ and $\signedBias_B(k)$ are both continuous and at \( k = \gamma \) we have
\begin{align*}
    \signedHet_B(\gamma) &= 0 \\
    \signedBias_B(\gamma) 
    &= \alpha_{1}^{\prime} \frac{\gamma (1 - \gamma) q_{1}}{(q_{0} + \gamma q_{1}) (1 - q_{0} - \gamma q_{1})}.
\end{align*}

Recall from Theorem~\ref{lem:decomposition} that \( \mu(x) = \mu^B(x) + \mu^H(x) \) and \( \tau(x) = \tau^B(x) + \tau^H(x) \) and hence,
\begin{align*}
  \expectation{\tau(X)} &= \expectation{\tau^B(X)} + \expectation{\tau^H(X)}, \\
  \expectation{\mu(X) \! + \! \tfrac{1}{2}\tau(X) \mid T \! = \! 1 } &= \expectation{\mu^B(X) \! + \! \tfrac{1}{2}\tau^B(X) \mid T \! = \! 1 } + \expectation{\mu^H(X) \! + \! \tfrac{1}{2}\tau^H(X) \mid T \! = \! 1 }, \\
  \expectation{\mu(X) \! - \! \tfrac{1}{2}\tau(X) \mid T \! = \! 0 } &= \expectation{\mu^B(X) \! - \! \tfrac{1}{2}\tau^B(X) \mid T \! = \! 0 } + \expectation{\mu^H(X) \! - \! \tfrac{1}{2}\tau^H(X) \mid T \! = \! 0 }.
\end{align*}
From~\eqref{eq:noBias} it follows that
\begin{multline*}
  \expectation{\mu(X) + \frac{1}{2}\tau(X) \mid T = 1 } - \expectation{\mu(X) - \frac{1}{2}\tau(X) \mid T = 0 } - \expectation{\tau(X)} \\
  = \expectation{\mu^B(X) + \frac{1}{2}\tau^B(X) \mid T = 1 } - \expectation{\mu^B(X) - \frac{1}{2}\tau^B(X) \mid T = 0 } - \expectation{\tau^B(X)}.
\end{multline*}
By Lemma~\ref{lem:bias_of_prediction} the right-hand side of this equation is equal to \( \signedBias_B(\gamma) \).
Because we assume that the difference-in-means estimator of the original model is biased, i.e., \( \expectation{\mu(X) + \frac{1}{2}\tau(X) \mid T = 1 } - \expectation{\mu(X) - \frac{1}{2}\tau(X) \mid T = 0 } - \expectation{\tau(X)} \neq 0 \), it follows that \( \signedBias_B(\gamma) \neq 0 \) along with \( \alpha_{1}^{\prime} \neq 0 \) and \( q_{1} \neq 0 \).
Moreover, as \( \signedBias_B(0) = \signedBias_B(1) = 0 \), it follows that \( \signedBias_B \) is non-constant and we will show that it has a global extremum at \( k = \gamma \). 

Considering the derivative of \( \signedBias_B \) for \( k > \gamma \) yields
\begin{align*}
    &\frac{\mathrm d}{\mathrm dk} \alpha_{1}^{\prime} \left(\frac{\gamma (1 - \gamma) q_{1}}{(q_{0} + \gamma q_{1}) (1 - q_{0} - \gamma q_{1})} - k \frac{\gamma (k - \gamma) q_{1}}{(k q_{0} + \gamma q_{1}) (k - k q_{0} - \gamma q_{1})}\right) \\
    &= -\alpha_{1}^{\prime} \left(\frac{(\gamma (k - \gamma) q_{1} + k \gamma q_{1}) (k q_{0} + \gamma q_{1}) (k - k q_{0} - \gamma q_{1})}{(k q_{0} + \gamma q_{1})^{2} (k - k q_{0} - \gamma q_{1})^{2}}\right. \\
    &\phantom{=} \left.- \frac{k \gamma (k - \gamma) q_{1} (q_{0} (k - k q_{0} - \gamma q_{1}) + (k q_{0} + \gamma q_{1}) (1 - q_{0}))}{(k q_{0} + \gamma q_{1})^{2} (k - k q_{0} - \gamma q_{1})^{2}}\right) \\
    &= -\alpha_{1}^{\prime} \gamma^{2} q_{1} \frac{(q_{0}^{2} + 2 q_{0} q_{1} - q_{0} - q_{1}) k^{2} + 2 \gamma q_{1}^{2} k - \gamma^{2} q_{1}^{2}}{(k q_{0} + \gamma q_{1})^{2} (k - k q_{0} - \gamma q_{1})^{2}}.
\end{align*}

Note that $\signedBias_B$ is contentiously differentiable in $(\gamma, 1)$. Hence, for a potential extremum \( k \) of $\signedBias_B$ in $(\gamma, 1)$, we would need
\begin{align*}
    0 &= -\alpha_{1}^{\prime} \gamma^{2} q_{1} \frac{(q_{0}^{2} + 2 q_{0} q_{1} - q_{0} - q_{1}) k^{2} + 2 \gamma q_{1}^{2} k - \gamma^{2} q_{1}^{2}}{(k q_{0} + \gamma q_{1})^{2} (k - k q_{0} - \gamma q_{1})^{2}} \\
    \iff 0 &= (q_{0}^{2} + 2 q_{0} q_{1} - q_{0} - q_{1}) k^{2} + 2 \gamma q_{1}^{2} k - \gamma^{2} q_{1}^{2} \\
    \iff k &= \frac{- 2 \gamma q_{1}^{2} \pm \sqrt{(2 \gamma q_{1}^{2})^{2} - 4 (q_{0}^{2} + 2 q_{0} q_{1} - q_{0} - q_{1}) (- \gamma^{2} q_{1}^{2})}}{2 (q_{0}^{2} + 2 q_{0} q_{1} - q_{0} - q_{1})},
    \intertext{but the discriminant is}
    &(2 \gamma q_{1}^{2})^{2} - 4 (q_{0}^{2} + 2 q_{0} q_{1} - q_{0} - q_{1}) (- \gamma^{2} q_{1}^{2}) \\
    &= 4 \gamma^{2} q_{1}^{4} + 4 \gamma^{2} q_{0}^{2} q_{1}^{2} + 8 \gamma^{2} q_{0} q_{1}^{3} - 4 \gamma^{2} q_{0} q_{1}^{2} - 4 \gamma^{2} q_{1}^{3} \\
    &= -4 \gamma^{2} q_{1}^{2} (q_{0} + q_{1}) (1 - q_{0} - q_{1}) < 0.
\end{align*}
Therefore, there are no local extrema for \( k \in (\gamma, 1) \). Similarly, there are also no local extrema within \( (0, \gamma) \). As \( \signedBias_B \) is continuous and non-constant with \( \signedBias_B(k) = 0 \) at \( k=0 \) and \( k=1 \), \( \signedBias_B(\gamma) \) must be a global extremum.

\paragraph{Heterogeneity}

After setting \( \alpha_0^{\prime\prime} = \beta_0^{\prime\prime} = 0 \), as explained above, the heterogeneity model is given by
\begin{align*}
    \mu^H(x) &= \alpha_{1}^{\prime\prime} \cdot \indicator{x \leq \gamma}, \\
    \tau^H(x) &= \beta_{1}^{\prime\prime} \cdot \indicator{x \leq \gamma}, \\
    p(x) &= q_{0} + q_{1} \cdot \indicator{x \leq \gamma},
\end{align*}
with \( \alpha_{1}^{\prime\prime} = -(1/2 - q_{0} - \gamma q_{1}) \beta_{1}^{\prime\prime} \), such that the bias calculated in Lemma~\ref{lem:bias_of_prediction} is 0.

For this model, the treatment effect obtained from the difference-in-means estimator tends to the correct average treatment effect, i.e., \( \hat{\tau}^{\infty}_{\parent} = \expectation{\tau(X)} = \gamma \beta_{1}^{\prime\prime} \) and the estimates in the child nodes converge to
\begin{align*}
  \hat{\tau}^{\infty}_{\leftchild} &= \begin{cases}
    \beta_{1}^{\prime\prime} & k \leq \gamma,\\
    \beta_{1}^{\prime\prime} \left(1 - (k - \gamma) \frac{(k q_{0} + \gamma q_{1}) (1 - q_{0}) - (q_{0} + \gamma q_{1}) \gamma q_{1}}{(k q_{0} + \gamma q_{1}) (k - k q_{0} - \gamma q_{1})}\right) & k > \gamma,
  \end{cases} \\
  \hat{\tau}^{\infty}_{\rightchild} &= \begin{cases}
    \beta_{1}^{\prime\prime} (\gamma - k) \frac{(q_{0} + \gamma q_{1}) (1 - q_{0} - \gamma q_{1}) - k (q_{0} + q_{1}) (1 - q_{0} - q_{1})}{((1 - k) q_{0} + (\gamma - k) q_{1}) ((1 - k) (1 - q_{0}) - (\gamma - k) q_{1})} & k \leq \gamma,\\
    0 & k > \gamma,
  \end{cases}
\end{align*}
which can be calculated by using \( \alpha_{1}^{\prime\prime} = -(1/2 - q_{0} - \gamma q_{1}) \beta_{1}^{\prime\prime} \) in the formulas given by Lemma~\ref{lem:population_predictions}.
The population signed splitting criteria evaluate to
\begin{align*}
    \signedHet_H(k)\! &= \sqrt{k (1-k)} (\hat{\tau}^{\infty}_{\leftchild} - \hat{\tau}^{\infty}_{\rightchild}), \\
    &= \begin{cases}
        \sqrt{k (1-k)} \left(\beta_{1}^{\prime\prime} - \beta_{1}^{\prime\prime} (\gamma - k) \frac{(q_{0} + \gamma q_{1}) (1 - q_{0} - \gamma q_{1}) - k (q_{0} + q_{1}) (1 - q_{0} - q_{1})}{((1 - k) q_{0} + (\gamma - k) q_{1}) ((1 - k) (1 - q_{0}) - (\gamma - k) q_{1})}\right) & k \leq \gamma,\\
        \sqrt{k (1-k)} \left(\beta_{1}^{\prime\prime} \left(1 - (k - \gamma) \frac{(k q_{0} + \gamma q_{1}) (1 - q_{0}) - (q_{0} + \gamma q_{1}) \gamma q_{1}}{(k q_{0} + \gamma q_{1}) (k - k q_{0} - \gamma q_{1})}\right) - 0\right) & k > \gamma,
    \end{cases} \\
    &= \begin{cases}
        \beta_{1}^{\prime\prime} \sqrt{k (1-k)} \left(1 - (\gamma - k) \frac{(q_{0} + \gamma q_{1}) (1 - q_{0} - \gamma q_{1}) - k (q_{0} + q_{1}) (1 - q_{0} - q_{1})}{((1 - k) q_{0} + (\gamma - k) q_{1}) ((1 - k) (1 - q_{0}) - (\gamma - k) q_{1})}\right) & k \leq \gamma,\\
        \beta_{1}^{\prime\prime} \sqrt{k (1-k)} \left(1 - (k - \gamma) \frac{(k q_{0} + \gamma q_{1}) (1 - q_{0}) - (q_{0} + \gamma q_{1}) \gamma q_{1}}{(k q_{0} + \gamma q_{1}) (k - k q_{0} - \gamma q_{1})}\right) & k > \gamma,
    \end{cases} \\[3mm]
    \signedBias_H(k)\! &= \hat{\tau}^{\infty}_{\parent} - k \hat{\tau}^{\infty}_{\leftchild} - (1-k) \hat{\tau}^{\infty}_{\rightchild} \\
    &= \begin{cases}
        \gamma \beta_{1}^{\prime\prime} - k \beta_{1}^{\prime\prime} - (1-k) \beta_{1}^{\prime\prime} (\gamma - k) \frac{(q_{0} + \gamma q_{1}) (1 - q_{0} - \gamma q_{1}) - k (q_{0} + q_{1}) (1 - q_{0} - q_{1})}{((1 - k) q_{0} + (\gamma - k) q_{1}) ((1 - k) (1 - q_{0}) - (\gamma - k) q_{1})} & k \leq \gamma,\\
        \gamma \beta_{1}^{\prime\prime} - k \beta_{1}^{\prime\prime} \left(1 - (k - \gamma) \frac{(k q_{0} + \gamma q_{1}) (1 - q_{0}) - (q_{0} + \gamma q_{1}) \gamma q_{1}}{(k q_{0} + \gamma q_{1}) (k - k q_{0} - \gamma q_{1})}\right) - (1-k) \cdot 0 & k > \gamma,
    \end{cases} \\
    &= \begin{cases}
        -\beta_{1}^{\prime\prime} (\gamma - k) \frac{k (1 - \gamma)^{2} q_{1}^{2}}{((1 - k) q_{0} + (\gamma - k) q_{1}) ((1 - k) (1 - q_{0}) - (\gamma - k) q_{1})} & k \leq \gamma,\\
        \beta_{1}^{\prime\prime} (k - \gamma) \frac{(1 - k) \gamma^{2} q_{1}^{2}}{(k q_{0} + \gamma q_{1}) (k - k q_{0} - \gamma q_{1})} & k > \gamma
    \end{cases}
\end{align*}
Note that $\signedHet_H(k)$ and $\signedBias_H(k)$ are both continuous and at \( k=\gamma \) we have
\begin{align*}
    \signedHet_H(\gamma) &= \beta_{1}^{\prime\prime} \sqrt{\gamma (1-\gamma)}, \\
    \signedBias_H(\gamma) &= 0.
\end{align*}
If the treatment effect of the original model \( \tau(x) \) is non-constant, it follows that \( 0 < \gamma < 1, \allowbreak \beta_1 \neq 0 \), and \( \beta_{1}^{\prime\prime} \neq 0 \). This implies \( \signedHet_H(\gamma) \neq 0 \) and, as \( \signedHet_H(0) = \signedHet_H(1) = 0 \), that \( \signedHet_H\) non-constant. We will now show that \( \signedHet_H(\gamma) \) is a local extremum.

Considering the derivative of \( \signedHet_H \) for $k < \gamma$ yields
\begin{align*}
    \frac{\mathrm{d}}{\mathrm{d}k} \signedHet_H(k) &= \beta_{1}^{\prime\prime} \sqrt{k (1-k)} \\ &\phantom{=}\cdot \Big(\frac{1 - 2 k}{2 k (1-k)} + \frac{(q_{0} + \gamma q_{1}) (1 - q_{0} - \gamma q_{1}) - k (q_{0} + q_{1}) (1 - q_{0} - q_{1})}{((1 - k) q_{0} + (\gamma - k) q_{1}) ((1 - k) (1 - q_{0}) - (\gamma - k) q_{1})}\Big) \\
    &- \beta_{1}^{\prime\prime} \sqrt{k (1-k)} (\gamma - k) \\ &\phantom{=}\cdot \Big(\begin{aligned}[t]
        &- \frac{(q_{0} + q_{1}) (1 - q_{0} - q_{1})}{((1 - k) q_{0} + (\gamma - k) q_{1}) ((1 - k) (1 - q_{0}) - (\gamma - k) q_{1})} \\
        &+ (q_{0} + q_{1}) \frac{(q_{0} + \gamma q_{1}) (1 - q_{0} - \gamma q_{1}) - k (q_{0} + q_{1}) (1 - q_{0} - q_{1})}{((1 - k) q_{0} + (\gamma - k) q_{1})^2 ((1 - k) (1 - q_{0}) - (\gamma - k) q_{1})} \\
        &+ (1 - q_{0} - q_{1}) \frac{(q_{0} + \gamma q_{1}) (1 - q_{0} - \gamma q_{1}) - k (q_{0} + q_{1}) (1 - q_{0} - q_{1})}{((1 - k) q_{0} + (\gamma - k) q_{1}) ((1 - k) (1 - q_{0}) - (\gamma - k) q_{1})^2} \\
        &+ \frac{1 - 2 k}{2 k (1-k)} \frac{(q_{0} + \gamma q_{1}) (1 - q_{0} - \gamma q_{1}) - k (q_{0} + q_{1}) (1 - q_{0} - q_{1})}{((1 - k) q_{0} + (\gamma - k) q_{1}) ((1 - k) (1 - q_{0}) - (\gamma - k) q_{1})}\Big)
    \end{aligned}
\end{align*}
and for $k > \gamma$
\begin{align*}
    \frac{\mathrm{d}}{\mathrm{d}k} \signedHet_H(k) &= \beta_{1}^{\prime\prime} \sqrt{k (1-k)}\left(\frac{1 - 2 k}{2 k (1-k)} - \frac{(k q_{0} + \gamma q_{1}) (1 - q_{0}) - (q_{0} + \gamma q_{1}) \gamma q_{1}}{(k q_{0} + \gamma q_{1}) (k - k q_{0} - \gamma q_{1})}\right) \\
    &+ \beta_{1}^{\prime\prime} \sqrt{k (1-k)} (k - \gamma) \Big(\begin{aligned}[t]
        &- \frac{q_{0} (1 - q_{0})}{(k q_{0} + \gamma q_{1}) (k - k q_{0} - \gamma q_{1})} \\
        &+ q_{0} \frac{(k q_{0} + \gamma q_{1}) (1 - q_{0}) - (q_{0} + \gamma q_{1}) \gamma q_{1}}{(k q_{0} + \gamma q_{1})^2 (k - k q_{0} - \gamma q_{1})} \\
        &+ (1 - q_{0}) \frac{(k q_{0} + \gamma q_{1}) (1 - q_{0}) - (q_{0} + \gamma q_{1}) \gamma q_{1}}{(k q_{0} + \gamma q_{1}) (k - k q_{0} - \gamma q_{1})^2} \\
        &- \frac{1 - 2 k}{2 k (1-k)} \frac{(k q_{0} + \gamma q_{1}) (1 - q_{0}) - (q_{0} + \gamma q_{1}) \gamma q_{1}}{(k q_{0} + \gamma q_{1}) (k - k q_{0} - \gamma q_{1})}\Big).
    \end{aligned}
\end{align*}
Note that $\signedHet_H(k)$ is continuously differentiable in $(0, \gamma)$ and in $(\gamma, 1)$.
The limits of these derivatives at \( \gamma \) are
\begin{align*}
    \lim_{k \to \gamma^{-}} \frac{\mathrm{d}}{\mathrm{d}k} \signedHet_H(k) &= \beta_{1}^{\prime\prime} \sqrt{\gamma (1-\gamma)} \frac{q_{0} (1 - q_{0}) + 2 \gamma^{2} q_{1}^{2}}{2 \gamma (1-\gamma) q_{0} (1-q_{0})} \\
    \lim_{k \to \gamma^{+}} \frac{\mathrm{d}}{\mathrm{d}k} \signedHet_H(k) &= -\beta_{1}^{\prime\prime} \sqrt{\gamma (1-\gamma)} \frac{(1-q_{0}-q_{1}) (q_{0}+q_{1}) + 2 (1-\gamma)^{2} q_{1}^{2}}{2 \gamma (1-\gamma) (q_{0} + q_{1}) (1-q_{0}-q_{1})}
\end{align*}

All of \( \gamma, (1-\gamma), q_{0}, (1 - q_{0}), (q_{0}+q_{1}) \) and \( (1-q_{0}-q_{1}) \) are assumed to be positive by Assumptions~A3 and \( \gamma \in (0, 1) \), as well as \( q_{1}^{2} \geq 0 \). Therefore, we have \( \sgn(\lim_{k \to \gamma^{-}} \frac{\mathrm{d}}{\mathrm{d}k} \signedHet_H(k)) = \sgn(\beta_{1}^{\prime\prime}) \) and \( \sgn(\lim_{k \to \gamma^{+}} \frac{\mathrm{d}}{\mathrm{d}k} \signedHet_H(k)) = -\sgn(\beta_{1}^{\prime\prime}) \). 
Also note that \( \sgn(\signedHet_H(\gamma)) = \sgn(\beta_{1}^{\prime\prime}) \).
Together, it follows that $\signedHet_H(\gamma)$ has a local extremum at $k = \gamma$.
\end{proof}
\end{document}